\documentclass{article}

     \usepackage[final]{neurips_2025}

\usepackage[utf8]{inputenc} 
\usepackage[T1]{fontenc}    
\usepackage{hyperref}       
\usepackage{url}            
\usepackage{booktabs}       
\usepackage{amsfonts}       
\usepackage{nicefrac}       
\usepackage{microtype}      
\usepackage{xcolor}         

\def \P{\mathbb{P}} 
\def \bfE {\mathbb{E}}

\newcommand{\indep}{\perp \!\!\! \perp}

\usepackage{amsmath}
\usepackage{amssymb}
\usepackage{mathtools}
\usepackage{amsthm}

\newtheorem{theorem}{Theorem}[section]
\newtheorem{proposition}[theorem]{Proposition}
\newtheorem{lemma}[theorem]{Lemma}

\theoremstyle{definition}
\newtheorem{definition}[theorem]{Definition}
\newtheorem{assumption}[theorem]{Assumption}
\theoremstyle{remark}

\usepackage{subfig}

\title{Learning Counterfactual Outcomes \\ Under  Rank Preservation}

\author{
Peng Wu$^{1}$, Haoxuan Li$^{2}$, Chunyuan Zheng$^{2}$, Yan Zeng$^{1}$, \\
 {\bf Jiawei Chen}$^{3}$, {\bf Yang Liu}$^{4}$, {\bf Ruocheng Guo}$^{5}$\thanks{Work done while at ByteDance Research, ruocheng.guo@bytedance.com.},~  {\bf Kun Zhang}$^{6,7}$  \vspace{10pt} \\
  $^1$Beijing Technology and Business University  \quad 
  $^2$Peking University  \\
  $^3$Zhejiang University   \quad  
   $^4$University of California, Santa Cruz \quad
  $^5$Intuit AI Research  \\
    $^6$Carnegie
Mellon University \quad 
    $^7$Mohamed 
bin Zayed University of Artificial Intelligence
}

\begin{document}

\maketitle

\begin{abstract}
Counterfactual inference aims to estimate the counterfactual outcome at the individual level given knowledge of an observed treatment and the factual outcome, with broad applications in fields such as epidemiology, econometrics, and management science. Previous methods rely on a known structural causal model (SCM) or assume the homogeneity of the exogenous variable and strict monotonicity between the outcome and exogenous variable. 
In this paper, we propose a principled approach for identifying and estimating the counterfactual outcome. 
We first introduce a simple and intuitive rank preservation assumption to identify the counterfactual outcome without relying on a known structural causal model.
Building on this, we propose a novel ideal loss for theoretically unbiased learning of the counterfactual outcome and further develop a kernel-based estimator for its empirical estimation. Our theoretical analysis shows that the rank preservation assumption is not stronger than the homogeneity and strict monotonicity assumptions, and shows that the proposed ideal loss is convex, and the proposed estimator is unbiased. Extensive semi-synthetic and real-world experiments are conducted to demonstrate the effectiveness of the proposed method. 
\end{abstract}

\section{Introduction}

Understanding causal relationships is a fundamental goal across various domains, such as epidemiology~\citep{Hernan-Robins2020}, 
econometrics~\citep{Imbens-Rubin2015}, and management science~\citep{kallus2020-policy}. Pearl and Mackenzie~\citep{Pearl-Mackenzie2018} define the three-layer causal hierarchy—association, intervention, and counterfactuals—to distinguish three types of queries with increasing complexity and difficulty~\citep{Bareinboim-etal2022}. 
Counterfactual inference, the most challenging level, aims to explore the impact of a treatment on an outcome given knowledge about a different observed treatment and the factual outcome. For example, given a patient who has not taken medication before and now suffers from a headache, we want to know whether the headache would have occurred if the patient had taken the medication initially. 
Answering such counterfactual queries can provide valuable instructions in scenarios such as credit assignment~\citep{Mesnard2021-ICML}, root-causal analysis~\citep{Budhathoki-etal2022}, attribution~\citep{Pearl-etal2016-primer, Dawid2022, Tian-Wu2025, Wu-Mao2025}, 
as well as fair and safe decision-making~\citep{imai2023principal, Wu-etal-2024-Harm, Wu-etal-2025-Safe}. 

Different from interventional queries, which are prospective and estimate the counterfactual outcome in a hypothetical world via only the observations obtained before treatment (as pre-treatment variables), counterfactual inference is retrospective and further incorporates the factual outcome (as a post-treatment variable) in the observed world. This inherent conflict between the hypothetical and the observed world poses a unique challenge and makes the counterfactual outcome generally unidentifiable, even in randomized controlled experiments (RCTs)~\citep{Bareinboim-etal2022, Pearl-etal2016-primer, Wu-etal-2024-Harm,  Ibeling-2020-AAAI}. 

For counterfactual inference, Pearl et al.~\citep{Pearl-etal2016-primer} proposed a three-step procedure (abduction, action, and prediction) to estimate counterfactual outcomes. However, it relies on the availability of structural causal models (SCMs) that fully describe the data-generating process~\citep{Brouwer2022, Xie-etal2023-attribution}. In real-world applications, the ground-truth SCM is likely to be unknown, and estimating it requires additional assumptions to ensure identifiability, such as linearity~\citep{shimizu2006linear} and additive noise~\citep{hoyer2008nonlinear, peters2014causal}. Unfortunately, these assumptions are hard to satisfy in practice and restrict the applicability. 


To tackle the above problems, several counterfactual learning approaches have been proposed with respect to different identifiability assumptions. For example, Lu et al.~\citep{Lu-etal2020-attribution}, Nasr-Esfahany et al.~\citep{Nasr-Esfahany-2023-attribution}, and Xie et al.~\citep{Xie-etal2023-attribution} established the identifiability of counterfactual outcomes based on homogeneity and strict monotonicity assumptions. The homogeneity assumption posits that the exogenous variable for each individual remains constant across different interventional environments, and the strict monotonicity assumption asserts that the outcome is a strictly monotone function of the exogenous variable given the features. In terms of counterfactual learning, \cite{Lu-etal2020-attribution} and \cite{Nasr-Esfahany-2023-attribution} adopted Pearl's three-step procedure that needs to estimate the SCM initially. In addition, \cite{Xie-etal2023-attribution} proposed using quantile regression to estimate counterfactual outcomes that effectively avoid the estimation of SCM. Nevertheless, it relies on a stringent assumption that the conditional quantile functions for different counterfactual outcomes come from the same model and it requires estimating a different quantile value for each individual, leading to a challenging bi-level optimization problem. 




 In this work, we propose a principled counterfactual learning approach with \emph{intuitive identifiability assumptions and theoretically guaranteed estimation methods}. {\bf On one hand}, 
 we introduce the simple and intuitive rank preservation assumption, positing that an individual's factual and counterfactual outcomes have the same rank in the corresponding distributions of factual and counterfactual outcomes for all individuals. 
 We establish the identifiability of counterfactual outcomes under the rank preservation assumption and show that it is slightly less restrictive than the homogeneity and monotonicity assumptions used in previous studies.

{\bf On the other hand}, we further propose a theoretically guaranteed method for unbiased estimation of counterfactual outcomes. The proposed estimation method has several desirable merits. First, unlike Pearl's three-step procedure, it does not necessitate a prior estimation of SCMs and thus relies on fewer assumptions than that in \cite{Lu-etal2020-attribution} and \cite{Nasr-Esfahany-2023-attribution}. Second, in contrast to the quantile regression method proposed by \cite{Xie-etal2023-attribution}, our approach neither restricts conditional quantile functions for different counterfactual outcomes to originate from the same model, nor does it require estimating a different quantile value for each unit. Third, we improve the previous learning approaches by adopting a convex loss for estimating counterfactual outcomes, which leads to a unique solution.


In summary, the main contributions are as follows: (1) We introduce the intuitive rank preservation assumption to identify the counterfactual outcomes with unknown SCM; (2) We propose a novel ideal loss for unbiased learning of the counterfactual outcome and further develop a kernel-based estimator for the ideal loss. In addition, we provide a comprehensive theoretical analysis for the proposed learning approach; (3) We conduct extensive experiments on both semi-synthetic and real-world datasets to demonstrate the effectiveness of the proposed method.

\section{Problem Formulation}
Throughout, capital letters represent random variables and lowercase letters denote their realizations. 
 
{\bf Structural Causal Model} (SCM, \cite{pearl2009causality}). An SCM $\mathcal{M}$ consists of a causal graph $\mathcal{G}$ and a set of structure equation models $\mathcal{F} = \{f_1, ..., f_p \}$.  
The nodes in $\mathcal{G}$ are divided into two categories: (a)  exogenous variables ${\bf U} = (U_1, ..., U_p)$,  which represent the environment during data generation, assumed to be mutually independent;    
(b) endogenous variables $\textbf{V} = \{V_1, ...,  V_p\}$, which denote the relevant features that we need to model in a question of interest. For variable $V_j$, its value is determined by a structure equation $V_j = f_j(PA_j, U_j), ~j=1, ..., p$, where $PA_j$ stands for the set of parents of $V_j$. 
SCM provides a formal language for describing how the variables interact and how the resulting distribution would change in response to certain interventions. Based on SCM, we introduce the counterfactual inference problem in the following. 
 {\bf Counterfactual Inference}. Suppose that we have three sets of variables denoted by $X, Y, {\bf E} \subseteq {\bf V}$,  counterfactual inference revolves around the question, ``given evidence  ${\bf E} = {\bf e}$, what would have happened if we had set $X$ to a different value $x'$?".  Pearl et al.~\cite{Pearl-etal2016-primer} propose using the three-step procedure to answer the problem: (a) {\bf Abduction}: determine the value of ${\bf U}$ according to the evidence ${\bf E} = {\bf e}$; (b) {\bf Action}: modify the model $\mathcal{M}$ by removing the structural equations for $X$ and replacing them with $X = x'$, yielding the modified model $\mathcal{M}_{x'}$; (c) {\bf Prediction}:  Use $\mathcal{M}_{x'}$ and the value of ${\bf U}$ to calculate the counterfactual outcome of $Y$. In this paper, we focus on estimating the counterfactual outcome for each individual.
To illustrate the main ideas, we formulate the common counterfactual inference problem within the context of the backdoor criterion.  

{\bf Problem Formulation.}  Let $\mathbf{V} = (Z,  X, Y)$, where $X$ causes $Y$, $Z$ affects both $X$ and $Y$, and the structure equation of $Y$ is given as 
	\begin{equation}  \label{eq1}
   Y = f_Y(X, Z, U_X).  
	 \end{equation} 
 Let $Y_{x'}$ denotes the potential outcome if we had set $X=x'$.   
 The counterfactual question, ``given evidence $(X = x, Z = z, Y = y)$ of an individual, what would have happened had we set $X = x'$ for this individual", is formally expressed as estimating $y_{x'}$, the realization of $Y_{x'}$ for the individual. Here, we adhere to the deterministic viewpoint of \cite{pearl2009causality} and \cite{Pearl-etal2016-primer}, treating the value of $Y_{x'}$ for each individual as a fixed constant.  
   According to Pearl’s three-step procedure, 
   given the evidence $(X = x, Z = z, Y = y)$ for an individual, the identifiability of its counterfactual value $y_{x'}$ can be achieved by determining the structural equation $f_Y$ and the value of $U_X$ for this individual. This is the key idea underlying most of the existing methods. 
   
 For clarity, we use $y_{x'}$ to denote the realization of the counterfactual outcome $Y_{x'}$ \emph{for a specific individual} with observed evidence $(X=x, Z=z, Y=y)$.

\section{Analysis of Existing Methods} \label{sec4}
In this section, we elucidate the challenges of counterfactual inference. Subsequently, we summarize the existing methods and shed light on their limitations.

\subsection{Challenges in Counterfactual Inference}  

The main challenge lies in that the counterfactual value $y_{x'}$ is generally not identifiable, even in randomized controlled experiments (RCTs). 
By definition, $y_{x'}$ is a quantity involving  two ``different worlds" at the same time: the observed world with $(X = x, Z=z, Y = y)$ and the hypothetical world where $X = x'$.  We only observe the factual outcome $Y_x = y$ but never observe the counterfactual outcome $Y_{x'}$, which is the fundamental problem in causal inference~\citep{Holland1986, Morgan-Winship-2015}. 
This inherent conflict prevents us from simplifying the expression of $y_{x'}$ to a do-calculus expression, making it generally unidentifiable, even in RCTs~\citep{Pearl-etal2016-primer}. 
Therefore, in addition to the widely used assumptions such as conditional exchangeability, overlapping, and consistency~\citep{Hernan-Robins2020},  counterfactual inference requires extra assumptions to ensure identifiability.  Essentially,  estimating $y_{x'}$ is equivalent to estimating the individual treatment effect $y_{x'} - y_{x}$, while the conditional average treatment effect (CATE) $\bfE[Y_{x'}  - Y_x | Z=z ]$  represents the ATE for a subpopulation with $Z=z$, overlooking the inherent heterogeneity in this subpopulation caused by the noise terms such as $U_X$~\citep{Wu-etal-2024-Harm, Albert-etal2005, Heckman1997, Djebbari-Smith2008, Ding-etal2019, Lei-Candes2021, Eli-etal2022}. 

\subsection{Summary of Existing Methods}  \label{sec3-2}
We summarize the existing methods in terms of identifiability assumptions and estimation strategies.   

We first present an equivalent expression of Eq. (\ref{eq1}) using  $(Y_x, Y_{x'})$. 
Eq. (\ref{eq1}) be reformulated as the following system 
\begin{equation*}  Y_x = f_Y(x, Z, U_{x}), ~  Y_{x'} = f_Y(x', Z, U_{x'}), \end{equation*} 
where $U_x$ and $U_{x'}$ denote the values of $U_X$ given $X=x$ and $X=x'$, respectively.  
The exogenous variable $U_X$ denotes the background and environment information induced by many unmeasured factors~\citep{Pearl-etal2016-primer}, and thus $U_{x}$ and $U_{x'}$ account for the heterogeneity of $Y_x$ and $Y_{x'}$ in the observed and hypothetical worlds, respectively. 
These two worlds may exhibit different levels of noise due to unmeasured factors~\citep{Heckman1997, Ding-etal2019, Chernozhukov2005}.   
For identification, previous work~\citep{Xie-etal2023-attribution, Lu-etal2020-attribution, Nasr-Esfahany-2023-attribution} relies on the key homogeneity and strict monotonicity assumptions.
\begin{assumption}[Homogeneity]  \label{assump1}
        $U_{x} = U_{x'}$. 
\end{assumption}

\begin{assumption}[Strict Monotonicity] \label{assump2} 
For any given $(x,z)$,  
	$Y_x = f_Y(x, z,  U_x)$  is a smooth and strictly monotonic function of $U_x$; or it is a bijective mapping from $U_x$ to $Y_x$. 
\end{assumption}

Assumption \ref{assump1} implies that the value of $U_X$ for each individual remains unchanged across $x$. Assumption \ref{assump2} implies that $Y_x$ is a strict monotonic function of $U_x$ in the subpopulation of $(X=x, Z=z)$. 
 In Assumption \ref{assump2}, the smoothness and strict monotonicity of $f_Y(x, z,  U_x)$ are akin to a bijective mapping of $Y_x$ and $U_x$ and serve the same purpose,  so we don't distinguish them in detail. 

\begin{lemma} \label{lem1}  
   Under Assumptions \ref{assump1}-\ref{assump2}, $y_{x'}$ is identifiable. 
\end{lemma}

For estimation of $y_{x'}$, 
 following Pearl’s three-step procedure,  
\cite{Lu-etal2020-attribution}  and \cite{Nasr-Esfahany-2023-attribution} initially estimate $f_Y$ and $U_X$ for each individual. However, estimating $f_Y$ and $U_X$ needs to impose extra assumptions, such as linearity~\citep{shimizu2006linear} and additive noise~\cite{peters2014causal}. 
In addition, \cite{Xie-etal2023-attribution} demonstrate that $y_{x'}$ corresponds to the $\tau^*$-th quantile of the distribution $\P(Y| X=x', Z=z)$, where $\tau^*$ is the quantile of $y$ in $\P(Y| X=x, Z=z)$ (See the proof of Lemma \ref{lem1} or Section \ref{sec4-1} for more details). Based on it, the authors uses quantile regression to estimate $y_{x'}$, which avoids the problem of estimating $f_Y$ and $U_X$. Nevertheless, this method fits a single model to obtain the conditional quantile functions for both  the counterfactual and factual outcomes. Thus, its validity relies on the underlying assumption that the conditional quantile functions of outcomes for different treatment groups stem from the same model. 
   In addition, it involves estimating a distinct quantile value for each individual before deriving the counterfactual outcomes, posing a challenging  bi-level optimization problem.

\section{Identification through Rank Preservation}  \label{sec5}

We introduce the rank preservation assumption for identifying $y_{x'}$.    
\emph{From a high-level perspective, identifying $y_{x'}$ essentially involves establishing the relationship between $Y_x$ and $Y_{x'}$ for each individual.} Pearl's three-step procedure achieves this by estimating $f_Y$ and $U_X$. 

\subsection{Rank Preservation Assumption}   \label{sec4-1}
Our identifiability assumption is based on Kendall’s rank correlation coefficient defined below. 

\begin{definition}[Kendall~\cite{Kendall1938}]  \label{def1}
Let $(x_1, y_1), ..., (x_n, y_n)$ be a set of observations of two random variables $(X, Y)$,  such that all the values of $x_{i}$ and $y_{i}$ are unique (ties are neglected for simplicity). 
 Any pair of $(x_i, y_i)$ and $(x_j, y_j)$, if $(x_j - x_i)(y_j - y_i) > 0$,  they are said to be concordant;  otherwise they are discordant.  
  The \emph{sample} Kendall rank correlation coefficient  is  defined as 
	\[ \rho_n(X, Y) = \frac{2}{n(n-1)} \sum_{1\leq i < j\leq n} \text{sign}( (x_i - x_j) (y_i - y_j)  ),      \]
where $\text{sign}(t) = -1, 0, 1$ for $t < 0$, $t =0$, $t > 0$, respectively. For any two random variables $(X, Y)$, we define $\rho(X, Y) = 1$, if $\rho_n(X, Y) =1$ for all integers $n\geq 2$.

\end{definition}

The $\rho_n(X, Y)$ also can be written as $2(N_c - N_d)/ n(n-1)$, where  $N_c$ is the number of concordant pairs, $N_d$ is the number of discordant pairs. It is easy to see that $-1 \leq  \rho_n(X, Y) \leq 1$ and if the agreement between the two rankings is perfect (i.e., perfect concordance),  $\rho_n(X, Y) = 1$. 


\begin{assumption}[Rank Preservation]   \label{assump3} 
 $\rho(Y_x, Y_{x'} | Z) = 1$.      
\end{assumption}

Assumption \ref{assump3} is a high-level condition that establishes a connection between $Y_x$ and $Y_{x'}$. This assumption is satisfied in many common scenarios, as illustrated below.
\begin{itemize}
  \item  Causal models with additive noise: $Y = g(X, Z) + U$ for  an arbitrary function $g$. 
  \item Heteroscedastic noise models: $Y = g(X, Z) + h(X, Z) U$  for arbitrary functions $g$ and $h$, with $h(X,Z) > 0$ denoting the conditional standard deviation of $Y$ given $(X, Z)$.    
\end{itemize}

For the individual with observation $(X =x, Z=z, Y= y)$, we denote $(y_x=y, y_{x'})$ as its true values of  ($Y_x, Y_{x'}$).  Assumption \ref{assump3} implies that for this individual,  
 its rankings of $y_{x}$ and $y_{x'}$ are the same in the distributions of $\P(Y_x | Z=z)$ and $\P( Y_{x'} | Z=z )$,  
 respectively. Therefore, we have
       \begin{equation} \label{eq4}
              \P( Y_x \leq y_x  | Z=z)  =  \P(  Y_{x'} \leq y_{x'}  | Z=z ).  
       \end{equation}
Since $y_x = y$ is observed and the distributions $\P(Y_x| Z=z) $ and $\P(Y_{x'}| Z=z)$ can be identified as $\P(Y | X=x, Z=z)$ and $ \P(Y | X=x', Z=z)$, respectively, by the backdoor criterion (i.e., $(Y_x, Y_{x'}) \indep X | Z$). Therefore, we have the following Proposition \ref{prop1} (see Appendix \ref{app-a} for proofs). 


 \begin{proposition} \label{prop1}    Under Assumption \ref{assump3}, $y_{x'}$ is identified as 
   the $\tau^*$-th quantile of  
    $\P(Y| X=x', Z=z)$, where $\tau^*$ is the quantile of $y$ in the distribution of $\P(Y | X=x, Z=z)$.
 \end{proposition}

Proposition \ref{prop1} shows that Assumption \ref{assump3} can serve as a substitute for Assumptions \ref{assump1}-\ref{assump2} in identifying $y_{x'}$.  Unlike Assumptions \ref{assump1}-\ref{assump2}, Assumption \ref{assump3} is simple and intuitive, as it directly links $Y_x$ and $Y_{x'}$ for each individual. 
To clarify the relationship between Assumption \ref{assump3} introduced by this work and Assumptions \ref{assump1}-\ref{assump2} from previous work, we present Proposition \ref{prop2} below.
\begin{proposition} \label{prop2}
  The proposed Assumption \ref{assump3} is strictly weaker than Assumptions \ref{assump1}-\ref{assump2}.
\end{proposition}


Proposition \ref{prop2} is intuitive, as correlation (Assumption \ref{assump3}) does not necessarily imply identity (Assumption \ref{assump1}). To illustrate, consider a SCM with $X \in \{0, 1\}$,   $Y_1 = Z + U_{1}$, $Y_0 = Z/2 + U_{0}$,  $U_{1} =  U_{0}^3$. In this case, $\rho(Y_0, Y_1 | Z) = 1$, but $U_1 \neq U_0$.
 Nevertheless, Assumption \ref{assump3} is only slightly weaker than Assumptions \ref{assump1}-\ref{assump2} by allowing $U_{x'} \neq U_{x}$. Specifically, we can show that if $U_x$ is a strictly monotone increasing function of $U_{x'}$, Assumption \ref{assump3} is equivalent to Assumption \ref{assump2}, see Appendix \ref{app-a} for proofs.  
\subsection{Further Relaxation of Strict Monotonicity}

In Definition \ref{def1}, we ignore ties for simplicity. However, 
 when the outcome $Y$ is discrete or continuous variables with tied observations, $\rho(Y_x, Y_{x'})$ will always be less than 1.  
To accommodate such cases, 
we introduce a modified version of the Kendall rank correlation coefficient given below. 
\begin{definition}[Kendall~\cite{Kendall1945}]  \label{def2} Let $(x_1, y_1), ..., (x_n, y_n)$ be the observations of two random variables $(X, Y)$, the modified Kendall rank correlation coefficient is define as   
	\begin{gather*} \tilde \rho_n(X, Y) =  \sum_{1\leq i < j\leq n} \frac{\text{sign}( (x_i - x_j) (y_i - y_j)  )}{ \sqrt{n(n-1)/2 - T_x} \cdot \sqrt{n(n-1)/2 - T_y}   }, \end{gather*}
	where $T_x$ is the number of tied pairs in $\{x_1, ..., x_n\}$ and  $T_y$ is the number of tied pairs in $\{y_1, ..., y_n\}$. 
	We define $\tilde \rho(X, Y) = 1$, if $\tilde \rho_n(X, Y) =1$ for all integers $n\geq 2$. 
\end{definition}

Compared with Definition \ref{def1}, one can see that $\tilde \rho(X,Y)$ adjusts $\rho(X, Y)$ by eliminating the ties in the denominator, and $\tilde \rho(X,Y)$ reduces to $\rho(X, Y)$ if there are no ties.   

\begin{assumption}[Rank Preservation] \label{assump4} 
 $\tilde \rho(Y_x, Y_{x'} | Z) = 1$.      
\end{assumption}

Assumption \ref{assump4} is less restrictive than Assumption \ref{assump3} as it accommodates broader data types of $Y$. To illustrate, consider a dataset with four individuals where the true values of $(Y_x, Y_{x'})$ are $(1, 1), (2, 1.5), (2, 1.5), (3, 2.5)$. In this scenario, $\sum_{1\leq i < j\leq n} \text{sign}( (y_{i,x} - y_{j,x}) (y_{i,x'} - y_{j,x'}) = 5$, $T_{Y_x} =1$, $T_{Y_{x'}}=1$, resulting in $\rho(Y_x, Y_{x'}) = 5/6$ and $\tilde \rho(Y_x, Y_{x'}) = 5/(\sqrt{6-1} \cdot \sqrt{6-1}) = 1$.

Assumption \ref{assump4} also guarantees the identifiability of $y_{x'}$.
 \begin{proposition} \label{prop3}    Under Assumption \ref{assump4}, the conclusion in Proposition \ref{prop1} also holds. 
 \end{proposition}

 

\section{Counterfactual Learning }\label{sec6}

We propose a novel estimation method for counterfactual inference. Suppose that $\{ (x_k, z_k, y_k): k = 1, ...,  N\}$ is a sample consisting of $N$ realizations of random variables $(X, Z, Y)$. For an individual, given its evidence $(X=x, Z=z, Y= y)$, we aim to estimate its counterfactual outcome $y_{x'}$, i.e., the realization of $Y_{x'}$ for this individual. 

\subsection{Rationale and Limitations of Quantile Regression}
  For estimating $y_{x'}$, 
  Xie et al.~\cite{Xie-etal2023-attribution} formulate it as the following bi-level optimization problem 
        \begin{align*}
                \tau^* ={}& \arg \min_{\tau} | f_{\tau}(x, z) - y |,  \quad
                f_{\tau}^* = \arg \min_{f} \frac 1 N \sum_{k=1}^N l_{\tau}( y_k - f(x_k, z_k) ), 
                        \end{align*}
 where $l_{\tau}(\xi) = \tau \xi \cdot \mathbb{I}(\xi \geq 0)+ (\tau - 1)\xi \cdot \mathbb{I}(\xi < 0)$ is the check function~\citep{Koenker1978}, the upper level optimization is to estimate $\tau^*$, the quantile of $y$ in the distribution $\P(Y | X=x, Z =z)$,  and the lower level optimization is to estimate the conditional quantile function $q(x, z; \tau) \triangleq \inf_{y}\{y: \P( Y \leq y | X=x, Z=z) \geq \tau \}$  for a given $\tau$. Then $y_{x'}$ can be estimated using $q(x', z; \tau^*)$. 

We define two conditional quantile regression functions,
        \begin{align*} 
              q_x(z; \tau)  \triangleq{}& \inf_{y}\{y: \P( Y_x \leq y | Z=z) \geq \tau \},\quad 
              q_{x'}(z; \tau)   \triangleq \inf_{y}\{y: \P( Y_{x'} \leq y | Z=z) \geq \tau \}. 
        \end{align*}
By Eq. (\ref{eq4}), $y_{x'}$ can be expressed as   $q_{x'}(z; \tau^*)$ with $\tau^*$ being  the quantile of $y$ in the distribution of $\P(Y_x | Z=z)$, i.e., $  \P( Y_{x} \leq y  | Z=z ) = \tau^*$. Lemma \ref{prop4} (see Appendix \ref{app-b} for proofs) shows the rationale behind employing the check function as the loss to estimate conditional quantiles.  
\begin{lemma} \label{prop4} 
We have that    \\ 
(i) $q_x(Z; \tau) = \arg\min_{f} \bfE[ l_{\tau}(Y_x - f(Z))]$ for any given $x$; \\ (ii)  $q(X, Z; \tau) = \arg\min_{f} \bfE[ l_{\tau}(Y - f(X, Z))]$.
\end{lemma}	 
 There are two major concerns with the estimation method of \cite{Xie-etal2023-attribution}. 
First, it only fits a single quantile regression model for $q(X, Z;\tau)$ to obtain estimates of $q_x(Z; \tau)$ and $q_{x'}(Z; \tau)$. When the two conditional quantile functions $q_x(Z; \tau)$ and $q_{x'}(Z; \tau)$ originate from different models, this method may yield inaccurate estimates.   
Second, it explicitly requires estimating the quantile $\tau^*$ for each individual before estimating the counterfactual outcome $y_{x'}$. 

Inspired by \cite{Firpo2007}, a simple improvement is to estimate $q_x(z;\tau)$ and $q_{x'}(z;\tau)$ separately. For example, for estimating $q_x(z;\tau)$, the associated loss function is given as 
 	\begin{align*}
	        R_x(f, \tau) = \frac 1  N \sum_{k=1}^N  \frac{\mathbb{I}(x_k=x) \cdot l_{\tau}(y_k - f(z_k) ) }{\hat p_x(z_k)},
	\end{align*} 
 where $p_x(z)= \P(X=x| Z=z)$ is the propensity score, $\hat p_x(z)$ is its estimate.
 Likewise, we could define $R_{x'}(f, \tau)$ by replacing $x$ with $x'$.   
 Then the estimation procedure for $y_{x'}$ involves four steps: (1) estimating $p_x(z)$; 
 (2) estimating $q_x(z; \tau)$ by minimizing $R_x(f, \tau)$ for a range of candidate values of $\tau$;  
(3) identifying the $\tau^*$ in the candidate set of $\tau$, that corresponds to the quantile of $y$ in the distribution $\P(Y |X = x, Z=z)$; (4) estimating $y_{x'}$ using $q_{x'}(z; \tau^*)$, where $q_{x'}(z; \tau^*)$ is obtained by minimizing $R_{x'}(f, \tau^*)$.  
Despite this four-step estimation method that allows $q_x(Z;\tau)$ and $q_{x'}(Z; \tau)$ to come from different models, it still needs to estimate a different $\tau^*$ for each individual. 

\subsection{Enhanced Counterfactual Learning Method}

To address the limitations  mentioned above 
in directly applying quantile regression and improve estimation accuracy, we propose a novel loss that produces an unbiased estimator of $y_{x'}$ for the individual with evidence $(X=x, Z=z, Y=y)$. 
 The proposed ideal loss is constructed as 
 \begin{align*}
     R_{x'}(t| x, z, & y) ={} \bfE \left [  | Y_{x'} - t |  ~ \big | ~ Z=z \right ] +  \bfE \left [ \text{sign}(Y_x - y)  ~ \big | ~  Z = z\right ] \cdot t,
 \end{align*}  
 which is a function of $t$ and the expectation operator is taken on the random variable of $(Y_x, Y_{x'})$ given $Z=z$. 
The proposed estimation method is based on Theorem \ref{thm-1}. 
\begin{theorem}[Validity of the Proposed Ideal Loss] \label{thm-1}
The loss $R_{x'}(t| x, z, y)$ is convex with respect to $t$ and is minimized uniquely at $t^*$, where $t^*$ is the solution satisfying 
       $$\P(Y_{x'} \leq t^* | Z=z) = \P(Y_x \leq y| Z=z).$$
\end{theorem}

Theorem \ref{thm-1} (see Appendix \ref{app-b} for proofs) implies that given the evidence $(X=x, Z=z, Y=y)$ for an individual, the counterfactual outcome $y_{x'}$
satisfies $y_{x'} = \arg \min_{t} R_{x'}(t| x, z, y)$ under Assumption \ref{assump4}. \emph{{\bf Importantly}, the loss $R_{x'}(t| x, z, y)$  neither
estimates the SCM a priori, nor restricts $q_x(z; \tau)$ and $q_{x'}(z; \tau)$ stem from the same model, and it does not need to estimate a different quantile value for each individual explicitly.}

To optimize the ideal loss $R_{x'}(t; x, z, y)$, we first need to estimate it, which presents two significant challenges: (1) $R_{x'}(t| x, z, y)$ involves both $Y_x$ and $Y_{x'}$, but for each unit, we only observe one of them; (2) The terms $ \bfE \left [  | Y_{x'} - t |  ~ \big | ~ Z=z \right ]$ and  $\bfE \left [ \text{sign}(Y_x - y)  ~ \big | ~  Z = z\right ]$ in  $R_{x'}(t| x, z, y)$ is conditioned on $Z=z$, and when $Z$ is a continuous variable with infinite possible values, it cannot be estimated by simply splitting the data based on $Z$. 
We employ inverse propensity score and kernel smoothing techniques to overcome these two challenges. 
 Specifically, we propose a kernel-smoothing-based estimator for the ideal loss, which is given as 
  \begin{align*}
        \hat R_{x'}(t| x, &z, y) ={} \frac{ \sum_{k=1}^N K_h(z_k -z) \frac{\mathbb{I}(x_k=x')}{\hat p_{x'}(z_k)} | y_k - t |  }{ \sum_{k=1}^N K_h(z_k -z) }  + \frac{ \sum_{k=1}^N K_h(z_k -z) \frac{\mathbb{I}(x_k=x)}{\hat p_{x}(z_k)} \cdot \text{sign}(y_k - y) }{  \sum_{k=1}^N K_h(z_k -z) }  \cdot t,
  \end{align*}
 where $h$ is a bandwidth/smoothing parameter, $K_h(u) = K(u/h)/h$, and $K(\cdot)$ is a symmetric kernel function~\citep{Fan-Gijbels1996, Li-Racine-2007, Li-etal2024-interference} that satisfies $\int K(u)du = 1$ and $\int u K(u)du =1$, such as Epanechnikov kernel $K(u) = 3(1-u^2) \cdot \mathbb{I}(|u|\leq 1)/4$  
 and Gaussian kernel $K(u)= \exp(-u^2/2)/\sqrt{2 \pi}$ for $u\in \mathbb{R}$. Then we can estimate $y_{x'}$ by minimizing $\hat R_{x'}(t; x, z, y)$ directly. 

\begin{proposition}[Consistency] \label{prop5} If $h \to 0$ as $N \to \infty$, 
 $\hat p_x(z)$ and $\hat p_{x'}(z)$ are consistent estimates of $p_x(z)$ and $p_{x'}(z)$, and the density function of $Z$ is differentiable, then  
$\hat R_{x'}(t| x, z, y)$ converges to $R_{x'}(t| x, z, y)$ 
 in probability.  
\end{proposition}


Proposition \ref{prop5} indicates that $\hat R_{x'}(t| x, z, y)$ is a consistent 
estimator of $R_{x'}(t| x, z, y)$, demonstrating the validity of the estimated ideal loss. The loss $\hat R_{x'}(t| x, z, y)$ is applicable only for discrete treatments due to the terms $\mathbb{I}(x_k = x')$ and $\mathbb{I}(x_k = x)$. However, it can be easily extended to continuous treatments, as detailed in Appendix \ref{app-d}.




\subsection{Further Theoretical Analysis}
We further analyze the properties of the proposed method, including the unbiasedness preservation of the ideal loss, the bias of the estimated loss $\hat R_{x'}(t| x, z, y)$ and its impact on the final estimate of the counterfactual outcome $y_{x'}$. 



First, we present the property of unbiasedness preservation. Let $R_{x'}^{\mathrm{weight}}(t| x, z, y)$ be  the weighted version of $R_{x'}(t| x, z, y)$, defined by   
\begin{gather*}\bfE \left [ w(X, Z)   | Y_{x'} - t |  \big |  Z=z \right ]+\bfE \left [ w(X, Z)\text{sign}(Y_x - y)   \big |   Z = z\right ] \cdot t, 
\end{gather*} 
where the weight $ w(X, Z)$ is an arbitrary function of $(X, Z)$. The following Theorem \ref{thm5-4} shows that 
$R_{x'}^{\mathrm{weight}}(t| x, z, y)$ is also valid for estimating the counterfactual outcome $y_{x'}$.   
\begin{theorem}[Unbiasedness Preservation] \label{thm5-4} 
The loss $R_{x'}^{\mathrm{weight}}(t| x, z, y)$ is convex in 
$t$ and is minimized uniquely at $t^*$ that satisfies 
       $\P(Y_{x'} \leq t^* | Z=z) = \P(Y_x \leq y| Z=z).$
\end{theorem}

From Theorem \ref{thm5-4} (see Appendix \ref{app-b} for proofs), if we set 
  $w(X, Z) = K_h(Z_k -z)/\{ \sum_{k=1}^N K_h(Z_k -z)/N\}$, 
 then the unbiasedness preservation property implicitly indicates that the proposed method is less sensitive to the choice of the kernel function. 
 

%

Then, we show the bias of the estimated loss $\hat R_{x'}(t| x, z, y)$.
\begin{proposition}[Bias of the Estimated Loss] \label{prop5-5} 
If $h \to 0$ as $N \to \infty$,  $p_x(z)/\hat p_x(z)$ and $p_{x'}(z)/\hat p_{x'}(z)$ are differentiable with respect to $z$, and the density function of $Z$ is differentiable, then the bias of $\hat R_{x'}(t| x, z, y)$, defined by $\bfE[ \hat R_{x'}(t|x,z,y) ] - R_{x'}(t|x,z,y)$, is given as 
    \begin{align*}
        \text{Bias}(&\hat R_{x'})  ={}
     \delta_{p_{x'}}  \bfE \left [  | Y_{x'} - t |  ~ \big | ~ Z=z \right ]  + \delta_{p_{x}}  \bfE \left [ \text{sign}(Y_x - y)  ~ \big | ~  Z = z\right ] \cdot t + O(h^2), 
    \end{align*} where $ \delta_{p_{x'}}  = (p_{x'}(z) - \hat p_{x'}(z))/\hat p_{x'}(z)$ and $\delta_{p_{x}}  = (p_{x}(z) - \hat p_{x}(z))/\hat p_{x}(z)$ are estimation errors of propensity scores.      
\end{proposition}

From Proposition \ref{prop5-5}, the bias of $\hat R_{x'}(t| x, z, y)$ consists of two components. The first is the estimation error of propensity scores. The second component arises from the kernel smoothing technique and is of order $h^2$. 
In addition, when $\hat p_x(z)$ and $\hat p_{x'}(z)$ are consistent estimators of $p_x(z)$ and $p_{x'}(z)$ (a weak condition), the bias converges to zero, and Proposition \ref{prop5-5} simplifies to Proposition \ref{prop5}.

Finally, we examine how the estimated $y_{x'}$, denoted as $\hat y_{x'} \triangleq \arg \min_t \hat R_{x'}(t| x, z, y)$, is influenced by the bias in the estimated loss.

\begin{theorem}[Bias of the Estimated Counterfactual Outcome] \label{thm5-6} Under the same conditions as in Proposition \ref{prop5-5}, $\hat y_{x'}$ converges to $\bar y_{x'}$  in probability, where $\bar y_{x'}$ satisfies that 
    \begin{gather*}  \P( Y_{x'} \leq \bar y_{x}|Z=z ) = \frac{ (2+2\delta_{p_{x}}) \P(Y_{x}\leq y| Z=z) + (\delta_{p_{x'}}- \delta_{p_{x}})  }{ 2 + 2\delta_{p_{x'}}  },
    \end{gather*}
 where $\bar y_{x'}$ may not equal to the true value $y_{x'}$ and their difference is the bias.    
\end{theorem}

From Theorem \ref{thm5-6}, one can see that the bias of $\hat y_{x'}$ mainly depends on the estimation error of propensity scores. When $\delta_{p_{x}} = \delta_{p_{x'}}$, the equation in Theorem \ref{thm5-6} reduces to equation (\ref{eq4}), and thus $\bar y_{x'} = y_{x'}$, i.e., no bias. 
Typically, both  $\delta_{p_{x}}$ and $\delta_{p_{x'}}$  are small due to the consistency of estimated propensity scores (a weak condition), and thus $\bar y_{x'}$ will close to $ y_{x'}$. 

\begin{table*}[t]
\centering
\caption{$\sqrt{\epsilon_{\text{PEHE}}}$ of individual treatment effect estimation on the simulated Sim-$m$ dataset, where $m$ is the dimension of $Z$.}
\resizebox{1\linewidth}{!}{
\begin{tabular}{l|ll|ll|ll|ll}
\toprule
          & \multicolumn{2}{c|}{Sim-5}                                                                                & \multicolumn{2}{c|}{Sim-10}   
          & \multicolumn{2}{c|}{Sim-20}          
          & \multicolumn{2}{c}{Sim-40}
          
              \\ \midrule
Methods   & \multicolumn{1}{c}{In-sample} & \multicolumn{1}{c|}{Out-sample} & \multicolumn{1}{c}{In-sample} & \multicolumn{1}{c|}{Out-sample} & \multicolumn{1}{c}{In-sample} & \multicolumn{1}{c|}{Out-sample} & \multicolumn{1}{c}{In-sample} & \multicolumn{1}{c}{Out-sample} \\
\cmidrule{1-9}
T-learner & 2.95 $\pm$ 0.02                & 2.66 $\pm$ 0.01               &  2.99 $\pm$ 0.01                & 3.17 $\pm$ 0.01                &  3.36 $\pm$ 0.02                        & 3.19 $\pm$ 0.03                        & 5.12 $\pm$ 0.02                        & 4.74 $\pm$ 0.04                       \\
X-learner & 2.94 $\pm$ 0.01                & 2.66 $\pm$ 0.01               & 2.98 $\pm$ 0.02               & 3.19 $\pm$ 0.02               &  3.31 $\pm$ 0.02                        &  3.21 $\pm$ 0.02                       &  5.08 $\pm$ 0.04                       & 4.77 $\pm$ 0.03                        \\
BNN       &  2.91 $\pm$ 0.08               & 2.64 $\pm$ 0.07           & 2.90 $\pm$ 0.11               & 3.08 $\pm$ 0.12   &  3.21 $\pm$ 0.13                        & 3.13 $\pm$ 0.16                        & 4.81 $\pm$ 0.10                         & 4.54 $\pm$ 0.09                        \\
TARNet    & 2.89 $\pm$ 0.07               & 2.64 $\pm$ 0.06            & 2.94 $\pm$ 0.07              & 3.16 $\pm$ 0.08                & 3.18 $\pm$ 0.07                         & 3.11 $\pm$ 0.07                        &  4.82 $\pm$ 0.07                       &  4.56 $\pm$ 0.07                       \\
CFRNet    & 2.88 $\pm$ 0.07                &  2.62 $\pm$ 0.06               & 2.94 $\pm$ 0.07                &  3.15 $\pm$ 0.08                & 3.15 $\pm$ 0.07                         & 3.08 $\pm$ 0.07                        & 4.71 $\pm$ 0.12                         & 4.45 $\pm$ 0.11                        \\
CEVAE    & 2.92 $\pm$ 0.27               & 2.65 $\pm$ 0.21            & 3.04 $\pm$ 0.27              & 3.11 $\pm$ 0.18                & 3.16 $\pm$ 0.17                         & 3.11 $\pm$ 0.17                        &  4.88 $\pm$ 0.23                       &  4.53 $\pm$ 0.20                         \\

DragonNet & 2.90 $\pm$ 0.08                & 2.63 $\pm$ 0.08               &  3.02 $\pm$ 0.07               & 3.25 $\pm$ 0.08                &  3.16 $\pm$ 0.11                        & 3.09 $\pm$ 0.10                        & 4.78 $\pm$ 0.11                        & 4.50 $\pm$ 0.12                        \\
DeRCFR    & 2.88 $\pm$ 0.06                & 2.61 $\pm$ 0.06               &  2.87 $\pm$ 0.05               & 3.07 $\pm$ 0.06                &  3.11 $\pm$ 0.07                        & 3.04 $\pm$ 0.06                        & 4.77 $\pm$ 0.11                         & 4.50 $\pm$ 0.10                        \\
DESCN     & 2.93 $\pm$ 0.11                & 2.66 $\pm$ 0.09               & 3.27 $\pm$ 0.81                & 3.46 $\pm$ 0.79                & 3.12 $\pm$ 0.20                         & 3.06 $\pm$ 0.20                        &  4.91 $\pm$ 0.37                        & 4.59 $\pm$ 0.35                        \\
ESCFR     & 2.87 $\pm$ 0.08               &  2.62 $\pm$ 0.07               & 2.94 $\pm$ 0.08                & 3.15 $\pm$ 0.09               & 3.03 $\pm$ 0.09                         & 3.06 $\pm$ 0.09                       & 4.71 $\pm$ 0.15                         &  4.43 $\pm$ 0.15                       \\
CFQP & 2.91 $\pm$ 0.09	& 2.67 $\pm$ 0.11	& 3.14 $\pm$ 0.30	& 3.40 $\pm$ 0.37 & 3.21 $\pm$ 0.12 &	3.18 $\pm$ 0.11 	& 4.93 $\pm$ 0.14	& 4.55 $\pm$ 0.13 \\
Quantile-Reg     & 2.80 $\pm$ 0.06               &  2.54 $\pm$ 0.05               & 2.78 $\pm$ 0.08                & 3.05 $\pm$ 0.09               & 2.92 $\pm$ 0.07                         & 3.01 $\pm$ 0.08                       & 4.39 $\pm$ 0.13                         &  4.12 $\pm$ 0.10                       \\
Ours      & \textbf{{2.45 $\pm$ 0.17}}                & \textbf{{2.28 $\pm$ 0.23}}               & \textbf{2.25 $\pm$ 0.07}                & \textbf{2.33 $\pm$ 0.07}                & \textbf{2.51 $\pm$ 0.07}                         &  \textbf{2.46 $\pm$ 0.06}                       &  \textbf{{3.74 $\pm$ 0.26}}                       & \textbf{{3.66 $\pm$ 0.21}}                        \\ \bottomrule
\end{tabular}}
\label{tab:sim}
\end{table*}


\section{Experiments}  \label{sec7}

\begin{table*}[]
\centering
\caption{$\sqrt{\epsilon_{\text{PEHE}}}$ of individual treatment effect estimation on the simulated Sim-$m$ dataset, where $m$ is the dimension of $Z$.}
\resizebox{1\linewidth}{!}{
\begin{tabular}{l|cc|cc|cc|cc}
\toprule  & \multicolumn{2}{c|}{Sim-80 ($\rho=0.3$)}   & \multicolumn{2}{c|}{Sim-80 ($\rho=0.5$)}  & \multicolumn{2}{c|}{Sim-40 ($\rho=0.3$)}   & \multicolumn{2}{c}{Sim-40 ($\rho=0.5$)}
\\ \midrule
Methods   & \multicolumn{1}{c}{In-sample} & \multicolumn{1}{c|}{Out-sample} & \multicolumn{1}{c}{In-sample} & \multicolumn{1}{c|}{Out-sample} & \multicolumn{1}{c}{In-sample} & \multicolumn{1}{c|}{Out-sample} & \multicolumn{1}{c}{In-sample} & \multicolumn{1}{c}{Out-sample} \\
\cmidrule{1-9}
TARNet     & 12.63 $\pm$  0.93 & 12.51 $\pm$  0.90 & 12.35 $\pm$  1.24 & 12.68 $\pm$  1.51 & 8.91 $\pm$  0.97 & 8.78 $\pm$  0.74 & 8.76 $\pm$  0.76 & 8.51 $\pm$  0.68 \\
DragonNet  & 12.50 $\pm$  0.75 & 12.36 $\pm$  0.80 & 12.71 $\pm$  1.29 & 13.02 $\pm$  1.54 & 8.83 $\pm$  0.90 & 8.73 $\pm$  0.72 & 8.62 $\pm$  0.70 & 8.39 $\pm$  0.53 \\
ESCFR     & 12.61 $\pm$  1.09 & 12.53 $\pm$  1.09 & 12.56 $\pm$  1.36 & 12.87 $\pm$  1.64 & 8.76 $\pm$  1.03 & 8.65 $\pm$  0.79 & 8.76 $\pm$  0.78 & 8.50 $\pm$  0.48 \\
X\_learner & 12.82 $\pm$  0.91 & 12.68 $\pm$  0.95 & 12.74 $\pm$  1.22 & 12.99 $\pm$  1.43 & 8.97 $\pm$  0.87 & 8.81 $\pm$  0.64 & 8.91 $\pm$  0.75 & 8.61 $\pm$  0.58 \\
Quantile-Reg & 11.59 $\pm$  0.94 & 11.57 $\pm$  0.97 & 11.59 $\pm$  1.26 & 11.91 $\pm$  1.47 & 8.05 $\pm$  0.73 & 8.08 $\pm$  0.75 & 7.74 $\pm$  0.73 & 7.58 $\pm$  0.73 \\
Ours       & \textbf{9.28 $\pm$  0.72}  & \textbf{9.28 $\pm$  0.72}  & \textbf{9.03 $\pm$  1.09}  & \textbf{9.27 $\pm$  0.97}  & \textbf{7.07 $\pm$  0.39} & \textbf{7.05 $\pm$  0.41} & \textbf{7.07 $\pm$  1.23} & \textbf{6.98 $\pm$  1.08} \\
\bottomrule
\end{tabular}}
\label{tab:sim_cov}
\end{table*}

\begin{figure*}[t]
\centering
\subfloat[Sim-10 (In-sample)]{
\begin{minipage}[t]{0.245\linewidth}
\centering
\includegraphics[width=1\textwidth]{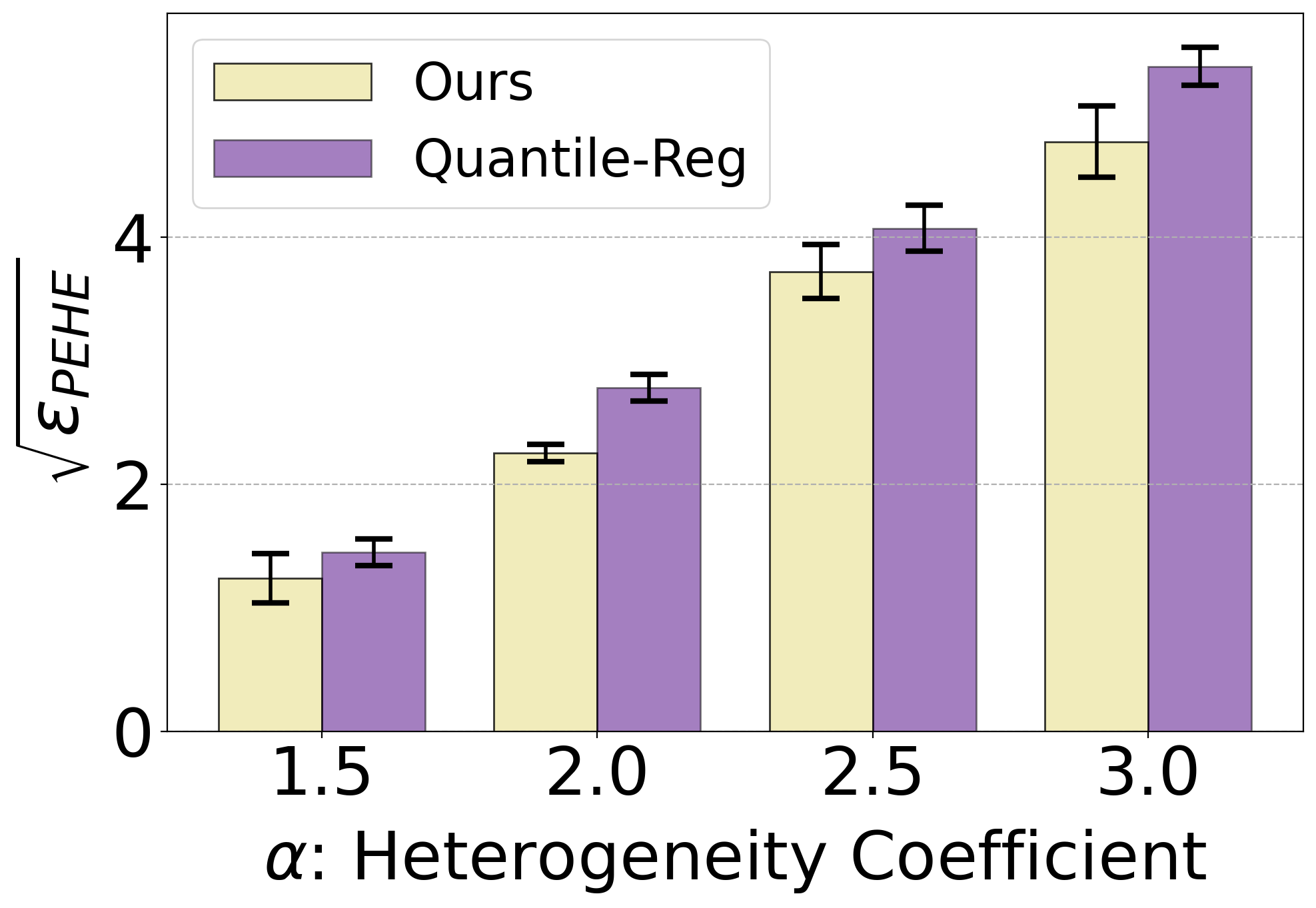}
\end{minipage}}%
\subfloat[Sim-10 (Out-sample)]{
\begin{minipage}[t]{0.245\linewidth}
\centering
\includegraphics[width=0.93\textwidth]{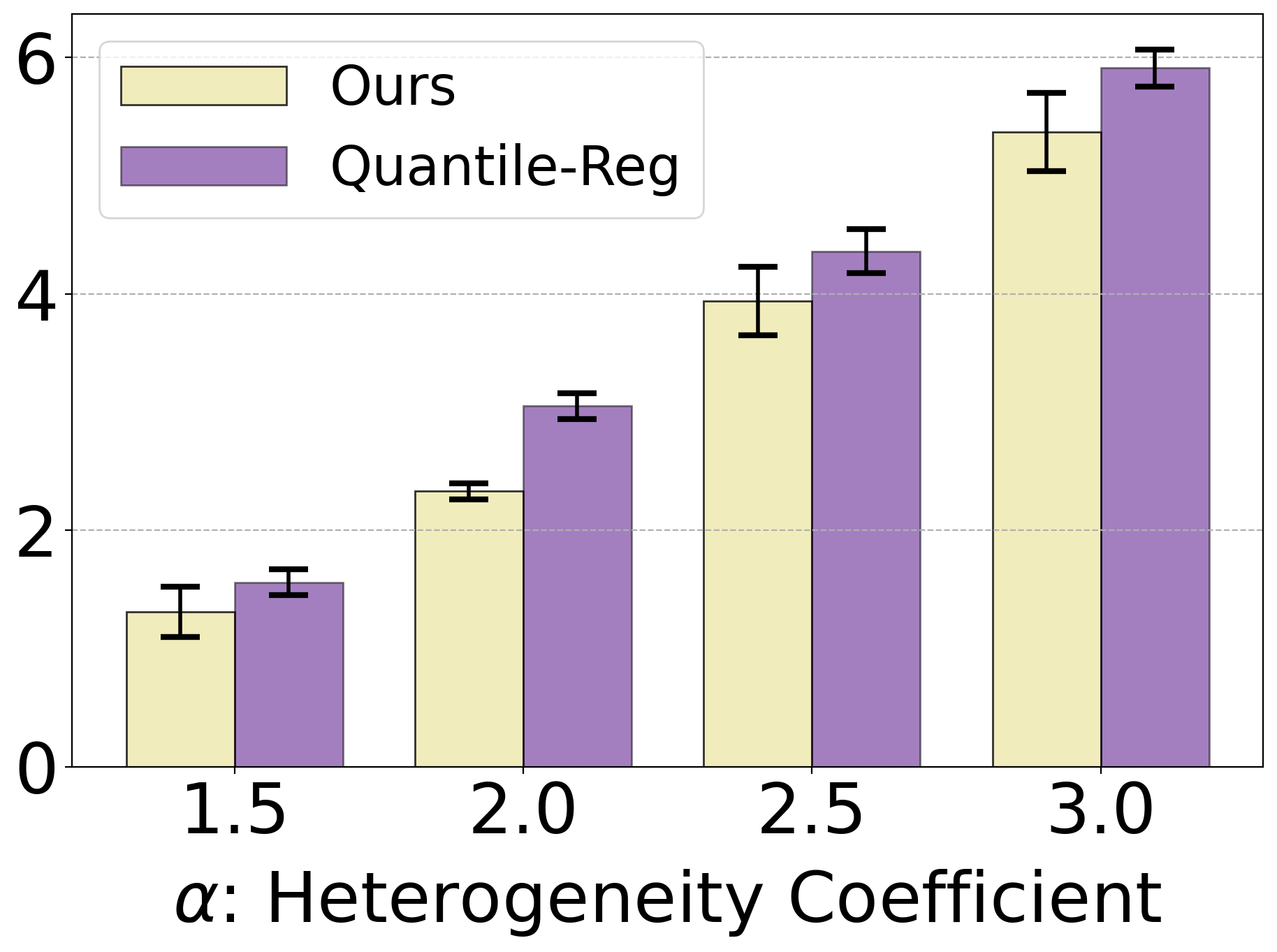}
\end{minipage}}%
\subfloat[Sim-40 (In-sample)]{
\begin{minipage}[t]{0.245\linewidth}
\centering
\includegraphics[width=0.95\textwidth]{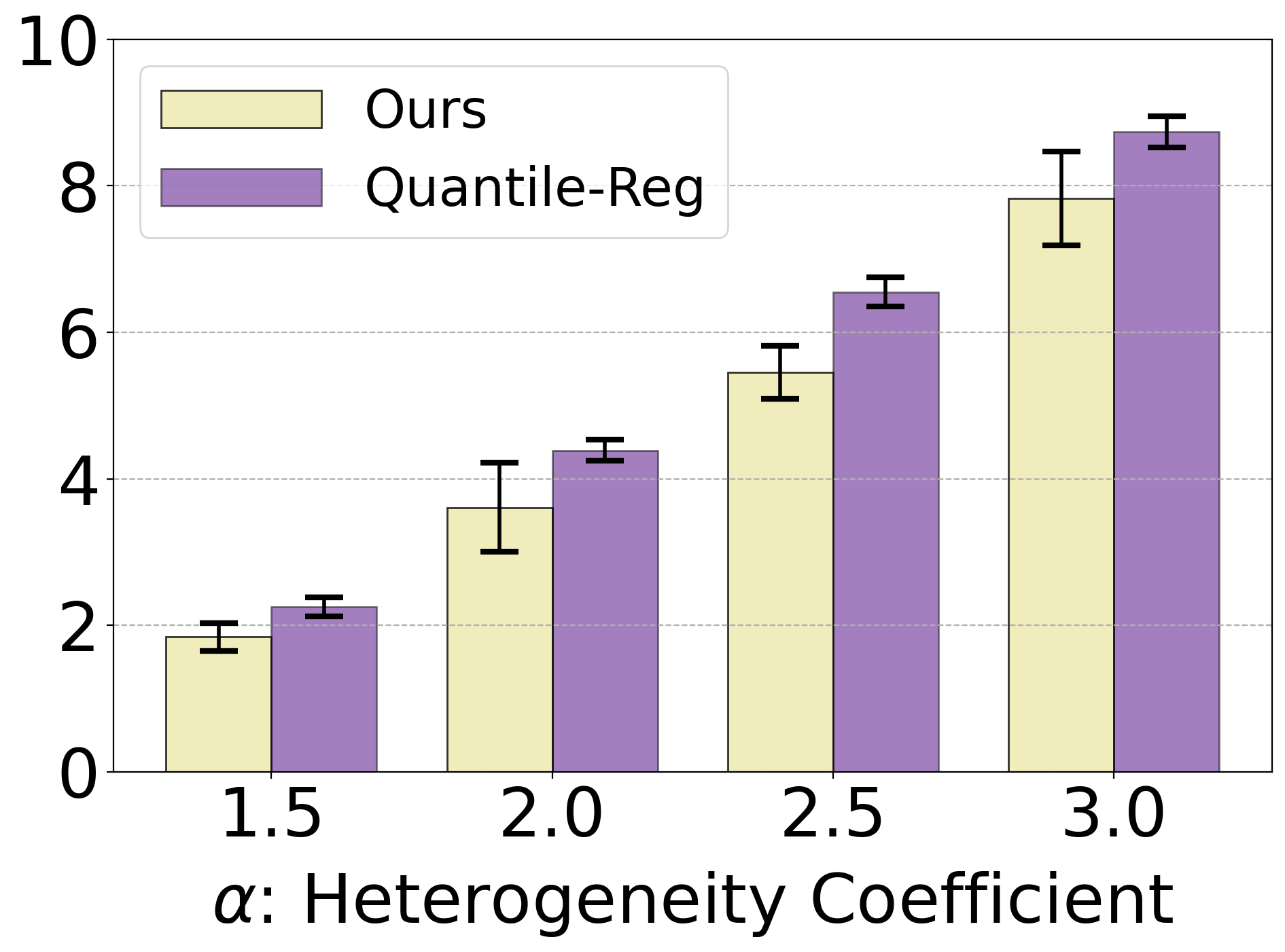}
\end{minipage}}%
\subfloat[Sim-40 (Out-sample)]{
\begin{minipage}[t]{0.245\linewidth}
\centering
\includegraphics[width=0.95\textwidth]{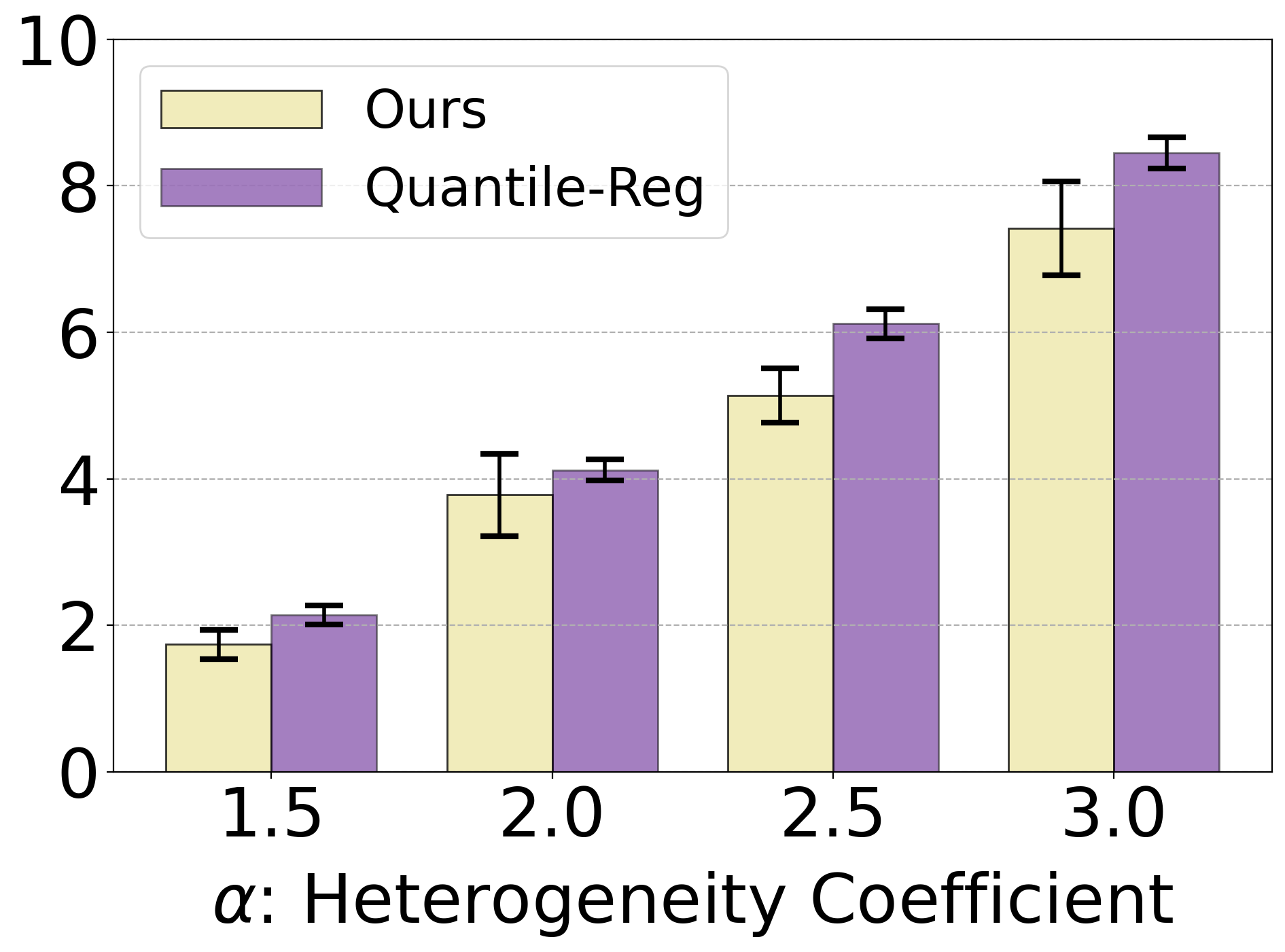}
\end{minipage}}%
\centering
\caption{Estimation performance of individual treatment effects under varying heterogeneity degrees.}
\label{fig:ratio}
\vspace{-12pt}
\end{figure*}

\begin{figure*}[t]
\centering
\subfloat[Sim-10 (In-sample)]{
\begin{minipage}[t]{0.245\linewidth}
\centering
\includegraphics[width=1\textwidth]{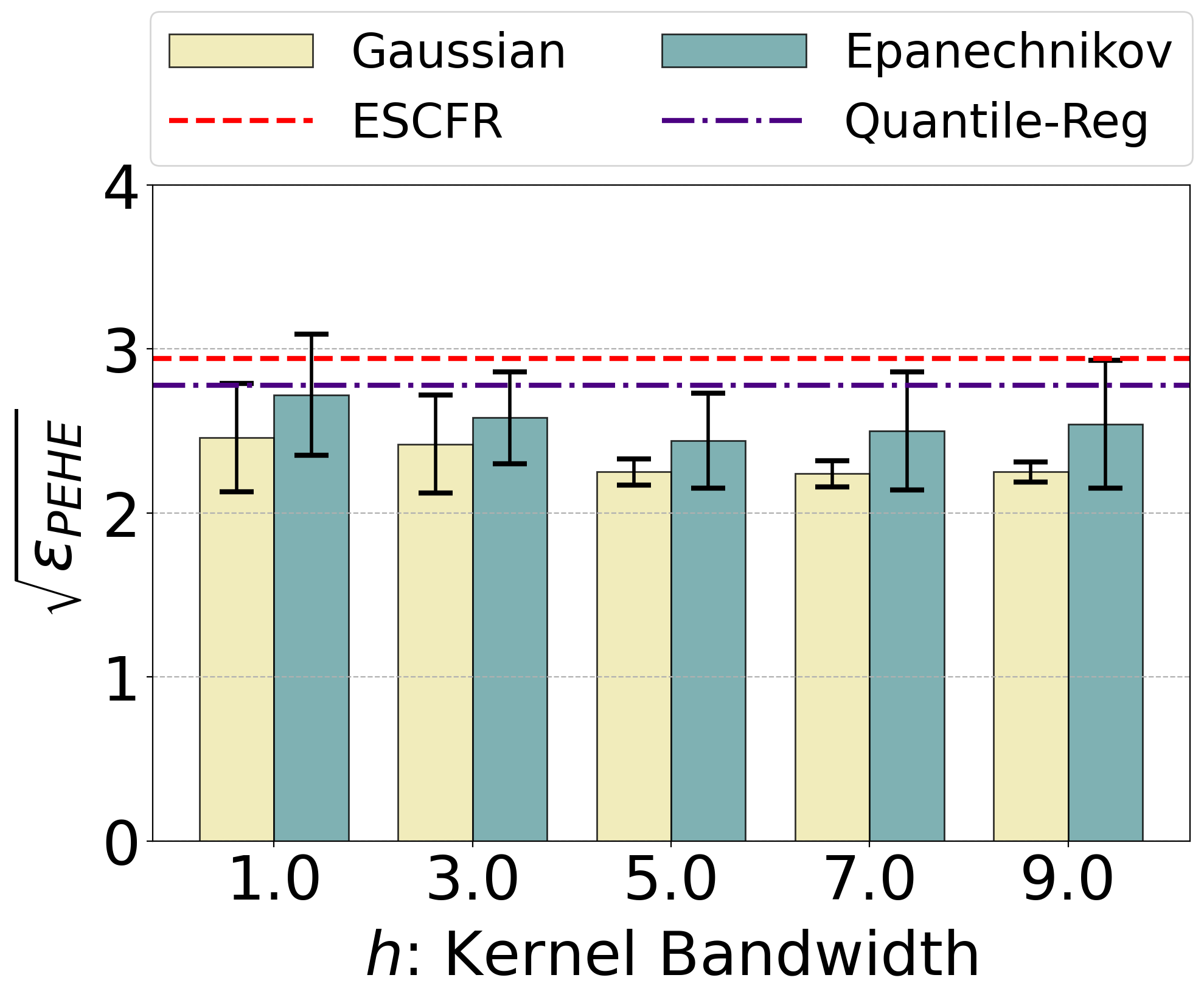}
\end{minipage}}%
\subfloat[Sim-10 (Out-sample)]{
\begin{minipage}[t]{0.245\linewidth}
\centering
\includegraphics[width=0.93\textwidth]{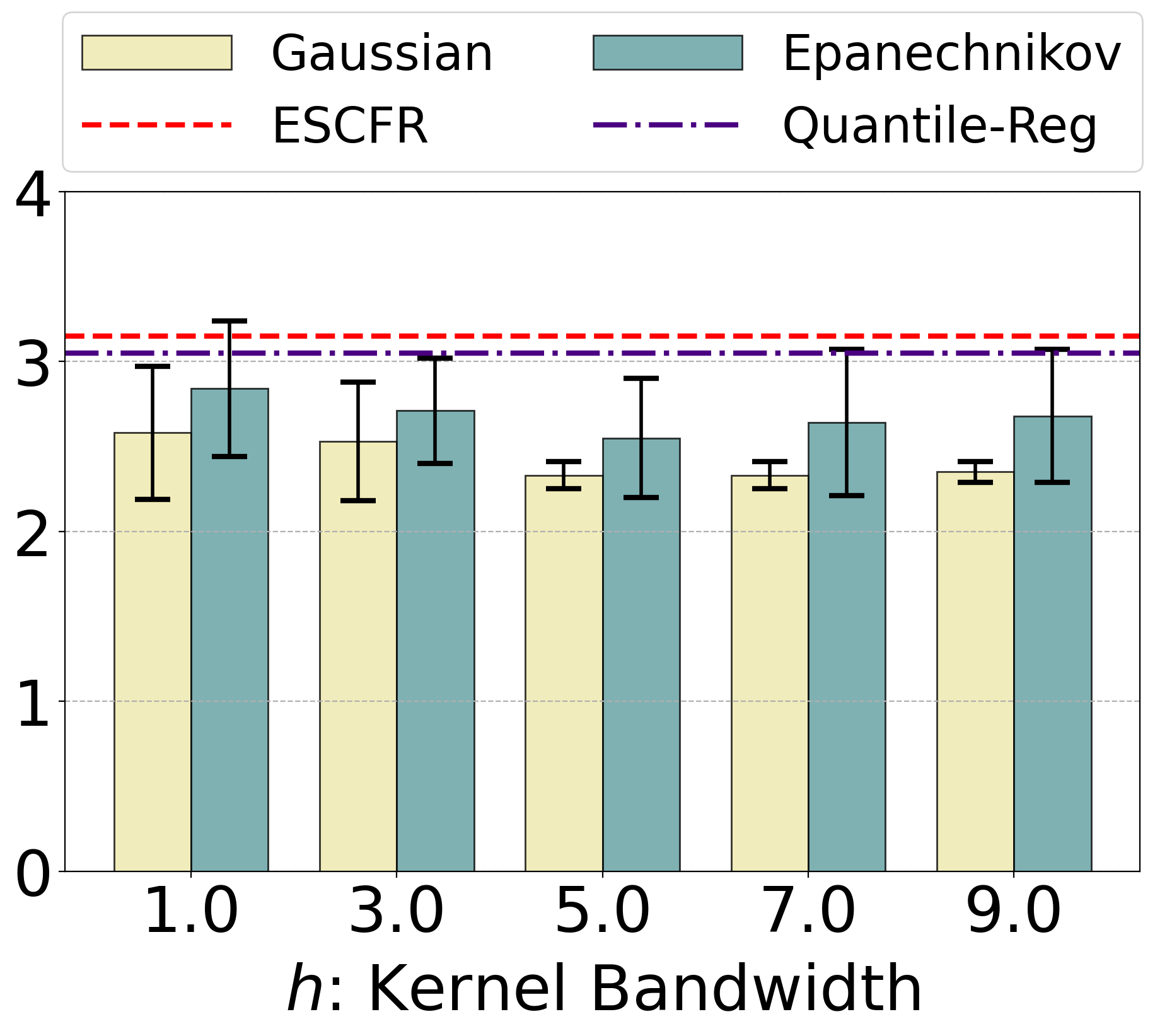}
\end{minipage}}%
\subfloat[Sim-40 (In-sample)]{
\begin{minipage}[t]{0.245\linewidth}
\centering
\includegraphics[width=0.93\textwidth]{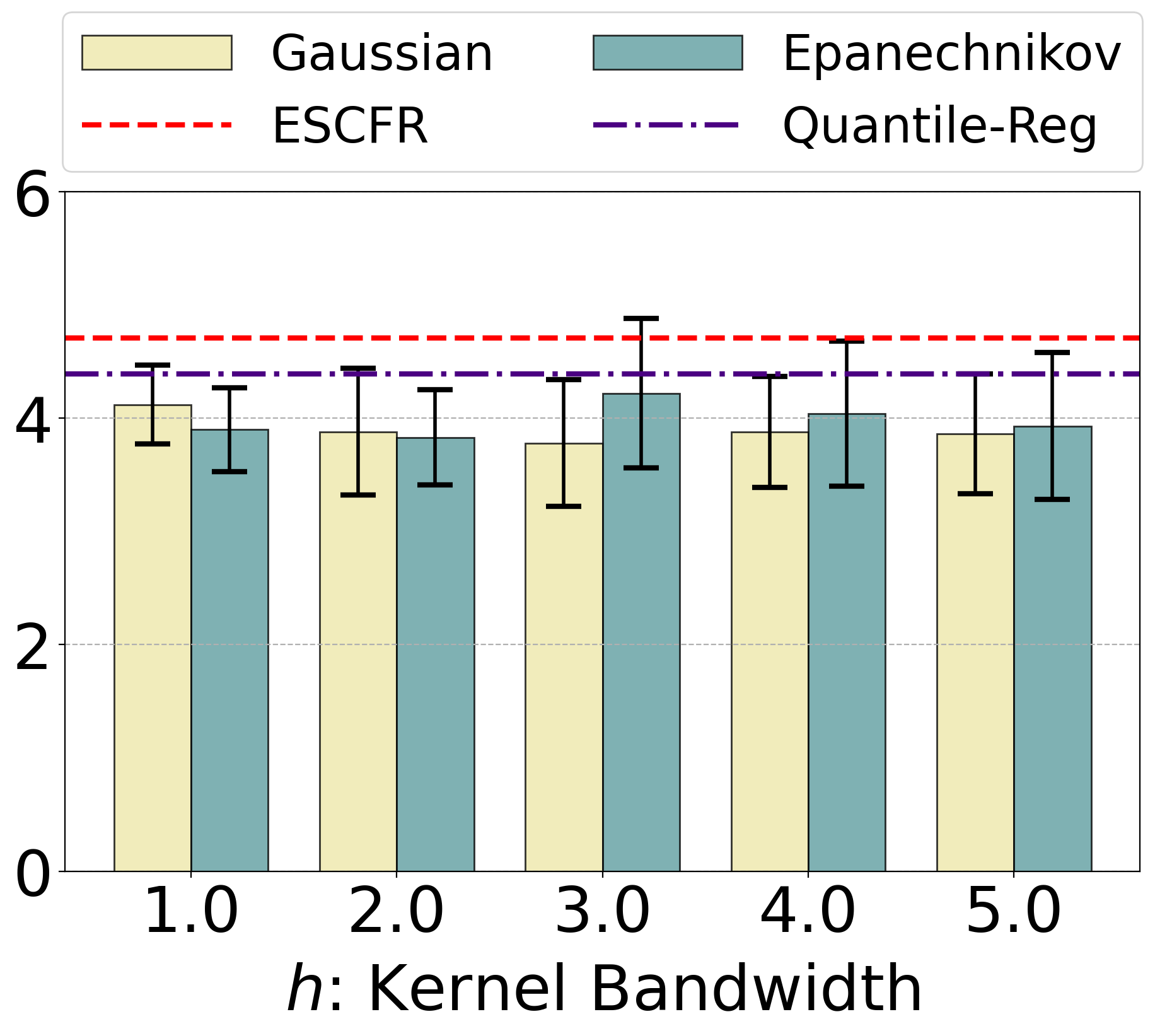}
\end{minipage}}%
\subfloat[Sim-40 (Out-sample)]{
\begin{minipage}[t]{0.245\linewidth}
\centering
\includegraphics[width=0.93\textwidth]{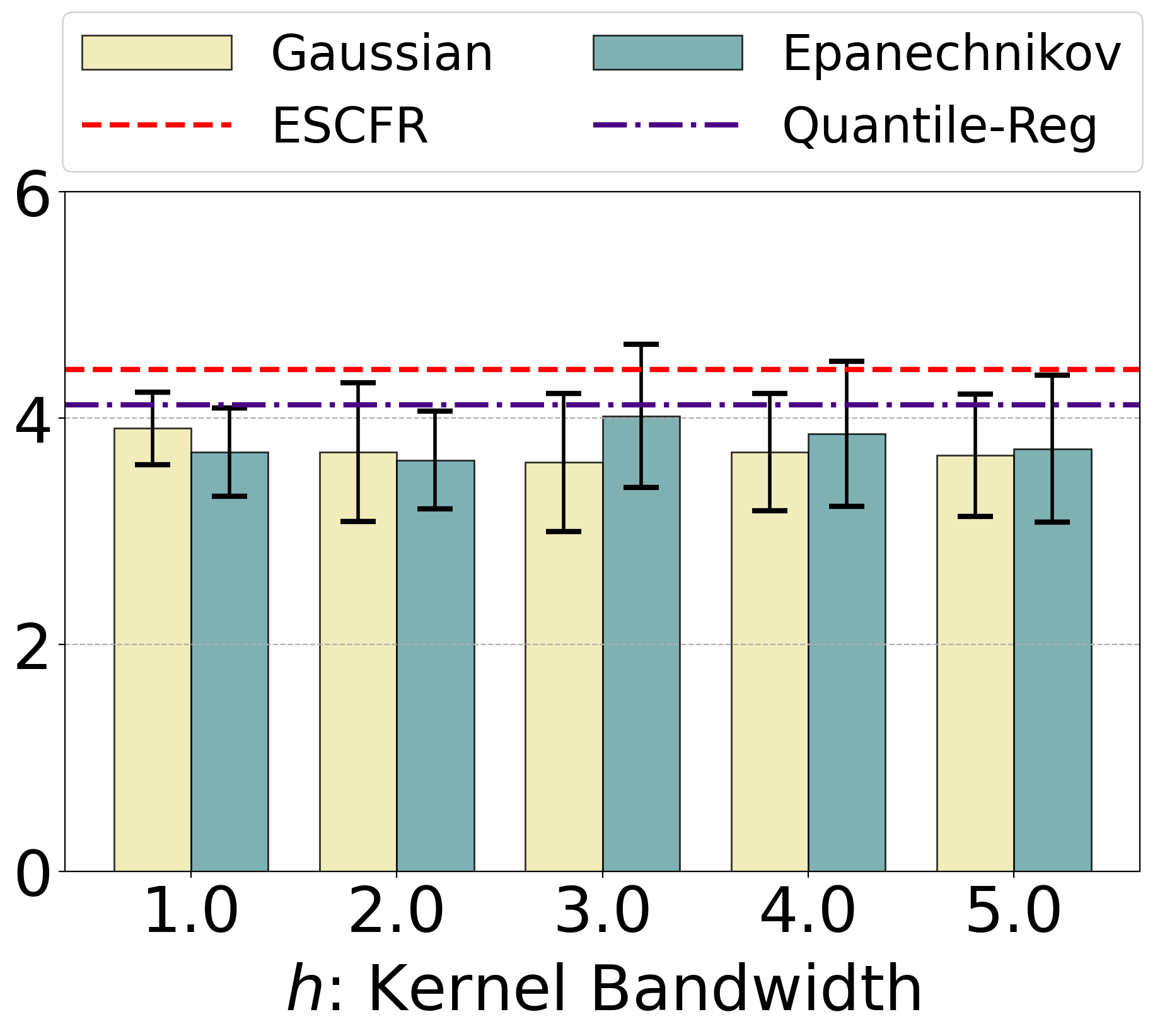}
\end{minipage}}%
\centering
\vspace{-3pt}
\caption{The estimation performance with different kernels and bandwidths.} 
\label{fig:kernel}
\vspace{-10pt}
\end{figure*}

\begin{table*}[t]
\centering
\caption{The experiment results on the \textsc{IHDP} dataset and \textsc{Jobs} dataset. The best result is bolded.}
\vspace{-6pt}
\resizebox{1\linewidth}{!}{
\begin{tabular}{l|llll|llll}
\toprule
          & \multicolumn{4}{c|}{\textsc{IHDP}}                                                                                & \multicolumn{4}{c}{\textsc{Jobs}}                                                                                \\ \cmidrule{2-9} 
          & \multicolumn{2}{c}{In-sample}                  & \multicolumn{2}{c|}{Out-sample}                  & \multicolumn{2}{c}{In-sample}                  & \multicolumn{2}{c}{Out-sample}                  \\ \midrule
Methods   & \multicolumn{1}{c}{$\sqrt{\epsilon_{\text{PEHE}}}$} & \multicolumn{1}{c}{$\epsilon_{\text{ATE}}$} & \multicolumn{1}{c}{$\sqrt{\epsilon_{\text{PEHE}}}$} & \multicolumn{1}{c|}{$\epsilon_{\text{ATE}}$} & \multicolumn{1}{c}{$R_{\text{Pol}}$} & \multicolumn{1}{c}{$\epsilon_{\text{ATT}}$} & \multicolumn{1}{c}{$R_{\text{Pol}}$} & \multicolumn{1}{c}{$\epsilon_{\text{ATT}}$} \\
\cmidrule{1-9}
T-learner & 1.49 $\pm$ 0.03                & 0.37 $\pm$ 0.05               & 1.81 $\pm$ 0.04                & 0.49 $\pm$ 0.04                & 0.31 $\pm$ 0.06                         & 0.16 $\pm$ 0.10                        &  0.27 $\pm$ 0.08                        & 0.20 $\pm$ 0.07                        \\
X-learner & 1.50 $\pm$ 0.02                & 0.21 $\pm$ 0.05               & 1.73 $\pm$ 0.03                & 0.36 $\pm$ 0.07                & 0.16 $\pm$ 0.04                         & 0.07 $\pm$ 0.05                        & 0.16 $\pm$ 0.03                         & 0.10 $\pm$ 0.09                        \\
BNN       & 2.09 $\pm$ 0.16                & 1.00 $\pm$ 0.23               & 2.37 $\pm$ 0.15                & 1.18 $\pm$ 0.19                & 0.15 $\pm$ 0.01                         &  0.08 $\pm$ 0.03                        & 0.16 $\pm$ 0.02                         & 0.13 $\pm$ 0.07                        \\
TARNet    & 1.52 $\pm$ 0.07                & 0.22 $\pm$ 0.13               & 1.78 $\pm$ 0.07                & 0.34 $\pm$ 0.18                & 0.17 $\pm$ 0.06                         & 0.06 $\pm$ 0.08                        & 0.18 $\pm$ 0.09                         & 0.10 $\pm$ 0.06                        \\
CFRNet    & 1.46 $\pm$ 0.06                & 0.17 $\pm$ 0.15               & 1.77 $\pm$ 0.06                & 0.32 $\pm$ 0.20                & 0.17 $\pm$ 0.03                         & \textbf{0.05 $\pm$ 0.03}                        & 0.19 $\pm$ 0.07                         & 0.10 $\pm$ 0.04                        \\
CEVAE    & 4.08 $\pm$ 0.88                & 3.67 $\pm$ 1.23               & 4.12 $\pm$ 0.91                & 3.75 $\pm$ 1.23                & 0.18 $\pm$ 0.05                         & 0.09 $\pm$ 0.03                        & 0.22 $\pm$ 0.08                        & 0.10 $\pm$ 0.09                        \\

DragonNet & 1.49 $\pm$ 0.08                & 0.22 $\pm$ 0.14               & 1.80 $\pm$ 0.06                & 0.29 $\pm$ 0.19                & 0.17 $\pm$ 0.06                         & 0.07 $\pm$ 0.07                        & 0.20 $\pm$ 0.08                         &  0.11 $\pm$ 0.09                       \\
DeRCFR    & 1.48 $\pm$ 0.06                & 0.25 $\pm$ 0.14               & 1.69 $\pm$ 0.06                & 0.25 $\pm$ 0.14                & 0.15 $\pm$ 0.02                         & 0.14 $\pm$ 0.04                        & 0.16 $\pm$ 0.04                         & 0.15 $\pm$ 0.11                        \\
DESCN     & 2.08 $\pm$ 0.98                & 0.74 $\pm$ 1.00               & 2.67 $\pm$ 1.45                & 1.04 $\pm$ 1.46                & 0.15 $\pm$ 0.02                         & 0.21 $\pm$ 0.14                        & 0.22 $\pm$ 0.16                         & 0.16 $\pm$ 0.04                        \\
ESCFR     & 1.46 $\pm$ 0.09                & 0.16 $\pm$ 0.16               & 1.73 $\pm$ 0.08                & 0.27 $\pm$ 0.16                & 0.14 $\pm$ 0.02                         & 0.10 $\pm$ 0.03                        & 0.15 $\pm$ 0.02                         & 0.10 $\pm$ 0.08                        \\
Quantile-Reg & 1.43 $\pm$ 0.05	& 0.14 $\pm$ 0.09	& 1.56 $\pm$ 0.03	& 0.18 $\pm$ 0.09 & 0.14 $\pm$ 0.01&	 0.06 $\pm$ 0.01	& 0.15 $\pm$ 0.01	& 0.07 $\pm$ 0.04 \\

CFQP & 1.47 $\pm$ 0.10	& 0.18 $\pm$ 0.17	& 1.48 $\pm$ 0.05	& 0.15 $\pm$ 0.08 & 0.15 $\pm$ 0.02 &	0.23 $\pm$ 0.15	& 0.16 $\pm$ 0.03	& 0.15 $\pm$ 0.07 \\

Ours      & \textbf{1.41 $\pm$ 0.02}                & \textbf{0.11 $\pm$ 0.10}               & \textbf{1.50 $\pm$ 0.06}                & \textbf{0.13 $\pm$ 0.08}                &  \textbf{0.08 $\pm$ 0.04}                         & 0.06 $\pm$ 0.02                        & \textbf{0.11 $\pm$ 0.05}                         & \textbf{0.05 $\pm$ 0.05}                        \\ 

\bottomrule
\end{tabular}}
\label{tab:real}
\vspace{-6pt}
\end{table*}

\subsection{Synthetic Experiment}
{\bf Simulation Process.} We generate the synthetic dataset by the following process. First, we sample the covariate $Z \sim \mathcal{N}\left(0, I_{m}\right)$ and the treatment $X \sim \operatorname{Bern}(\pi(Z))$, where $\operatorname{Bern}(\cdot)$ is the Bernoulli distribution with probability $ \pi(Z) = \mathbb{P}(X = 1 \mid Z) = \sigma(W_x \cdot Z)$, $\sigma(\cdot)$ is the sigmoid function, and $W_x \sim \operatorname{Unif}(-1, 1)^{m}$, $\operatorname{Unif}(\cdot)$ is the uniform distribution. Then, we sample the noise $U_0 \sim \mathcal{N}\left(0, 1\right)$ and $U_1 = \alpha \cdot U_0$ to consider the heterogeneity of the exogenous variables, where $\alpha$ is the hyper-parameter to control the heterogeneity degree. Finally, we simulate $ Y_1 = W_y \cdot Z + U_{1}$ and $Y_0 =  W_y \cdot Z/\alpha + U_{0}$ with $W_y \sim \mathcal{N}(0,I_{m})$. We generate 10,000 samples with 63/27/10 train/validation/test split and vary $m \in \{5, 10, 20, 40\}$ in our synthetic experiment.

{\bf Baselines and Evaluation Metrics.} 
The competing baselines includes: 
T-learner~\citep{kunzel2019metalearners}, X-learner~\citep{kunzel2019metalearners}, BNN~\citep{johansson2016learning}, TARNet~\citep{shalit2017estimating}, CFRNet~\citep{shalit2017estimating}, CEVAE~\citep{louizos2017causal}, DragonNet~\citep{shi2019adapting}, DeRCFR~\citep{wu2022learning}, DESCN~\citep{zhong2022descn}, ESCFR~\citep{wang2023optimal}, CFQP~\citep{Brouwer2022},
and Quantile-Reg~\citep{Xie-etal2023-attribution}. 
We evaluate the individual treatment effect estimation using the \emph{individual level Precision in Estimation
of Heterogeneous Effects} (PEHE): $$\epsilon_{\text{PEHE}} =
{\frac{1}{N} \sum_{i=1}^N [(\hat {Y}_i(1) - \hat {Y}_i(0)) - (Y_i(1) - Y_i(0))]^2},$$
where $\hat {Y}_i(1)$ and $\hat {Y}_i(0)$ are the predicted values for the corresponding true potential outcomes of unit $i$.  
 It is noteworthy that $\epsilon_{\text{PEHE}}$ is tailored for individual-level evaluation and counterfactual estimation, which is different from the common metric~\citep{shalit2017estimating} given by  
     $\frac{1}{N} \sum_{i=1}^N[  (\hat \mu_1(X_i) - \hat \mu_0(X_i)) - ( \mu_1(X_i) -  \mu_0(X_i)  )]^2,$
where  $\mu_1(X_i) -  \mu_0(X_i) := \bfE[ Y(1)|X ]- \bfE[ Y(0)|X ]$ are the true CATE, and $\hat \mu_1(X_i) - \hat \mu_0(X_i)$ is its estimate. 
Both in-sample and out-of-sample performances are reported in our experiments. For more implementation details of the proposed method, please refer to 
 Appendix \ref{app-e}.  

{\bf Performance Analysis.} The results of estimation performance are shown in Table \ref{tab:sim}. Our method stably outperforms all baselines with varying covariate dimensions $m$, demonstrating the effectiveness of the proposed method. In addition, we investigate our method performance with violated assumptions on rank and uncorrelated covariates. Specifically, we modified the data generation process to explore the performance of our method under correlated covariates. Specifically, we sample the covariate $Z \sim \mathcal{N}\left(0, \Sigma_{m}\right)$, where the $\rho_{ij}$ in $\Sigma_m$ is $\max(0.01, \rho^{|i-j|})$. The results are shown in Table \ref{tab:sim_cov}. The results show that our method still outperforms the baseline methods. See Appendix \ref{app-e} for the results with rank assumption violated. Moreover, we further explore the effect of heterogeneity degrees on the performance of the proposed method, as shown in Figure \ref{fig:ratio}, from which one can see that as the heterogeneity degree increases, our method stably outperforms the Quantile-Reg in terms of PEHE. Finally, we examine the effect of different kernels and bandwidths, as shown in Figure \ref{fig:kernel}, our method stably outperforms the Quantile-Reg and ESCFR methods with different kernels and bandwidths.


\subsection{Real-World Experiment}

{\bf Dataset and Preprocessing.} Following previous studies~\cite{shalit2017estimating, louizos2017causal, yoon2018ganite, yao2018representation}, we conduct experiments on semi-synthetic dataset \textsc{IHDP} and real-world dataset \textsc{Jobs}. The \textsc{IHDP} dataset~\citep{hill2011bayesian} is constructed from the Infant Health and Development Program (IHDP) with 747 individuals and 25 covariates. The \textsc{Jobs} dataset~\citep{lalonde1986evaluating} is based on the National Supported Work program with 3,212 individuals and 17 covariates. We follow \cite{shalit2017estimating} to split the data into training/validation/testing set with ratios 63/27/10 and 56/24/20 with 100 and 10 repeated times on the \textsc{IHDP} and the \textsc{Jobs} datasets, respectively.


{\bf Evaluation Metrics.} Following previous studies~\citep{shalit2017estimating, louizos2017causal, yao2018representation}, besides $\epsilon_{\text{PEHE}}$, we also use the absolute error in \emph{Average Treatment Effect} (ATE) for evaluation, which is defined as $\epsilon_{\text{ATE}}=\frac{1}{N}| \sum_{i=1}^N ((\hat {Y}_i(1) - \hat {Y}_i(0)) - (Y_i(1) - Y_i(0)))|.$
We use $\sqrt{\epsilon_{\text{PEHE}}}$ and $\epsilon_{\text{ATE}}$ to evaluate performance on the \textsc{IHDP} dataset. For the \textsc{Jobs} dataset, since one of the potential outcomes is not available, we evaluate the performance using the absolute error in \emph{Average Treatment effect on the Treated} (ATT) as $\epsilon_{\text{ATT}} = |\text{ATT} - \frac{1}{|T|} \sum_{i \in T} (\hat {Y}_i(1) - \hat {Y}_i(0)|$
with $\text{ATT} = |\frac{1}{|T|} \sum_{i \in T} {Y}_i - \frac{1}{|C \cap E|} \sum_{i \in C \cap E} {Y}_i|$. We also use the policy risk ${R}_{\text{Pol}}= 1-(\mathbb{E}[Y(1) \mid \hat {Y}(1) - \hat {Y}(0) >0, X=1] \cdot \mathbb{P}(\hat {Y}(1) - \hat {Y}(0)>0)+\mathbb{E}[Y(0) \mid \hat {Y}(1) - \hat {Y}(0) \leq 0, X=0] \cdot \mathbb{P}(\hat {Y}(1) - \hat {Y}(0) \leq 0))$, where $T, C, E$ are the indexes of  treatment sample set, control sample set, and randomized sample set, respectively. 

{\bf Performance Comparison.}
The experiment results are shown in Table \ref{tab:real}. Similar to the synthetic experiment, the Quantile-Reg method still achieves the most competitive performance compared to the other baselines. Our method stably outperforms all the baselines on both the semi-synthetic dataset \textsc{IHDP} and the real-world dataset \textsc{Jobs}, especially in the out-sample scenario. This provides the empirical evidence of the effectiveness of our method.

\section{Related Work}
{\bf Conditional Average Treatment Effect (CATE).} 
 CATE also referred to as heterogeneous treatment effect, represents the average treatment effects on subgroups categorized by covariate values, and plays a central role in areas such as precision medicine~\cite{Kosorok+Laber:2019} and policy learning~\cite{Dudik2011-ICML}. Benefiting from recent advances in machine learning, many methods have been proposed for estimating CATE, including matching methods~\cite{rosenbaum1983central,schwab2018perfect,yao2018representation}, tree-based methods~\cite{chipman2010bart,wager2018estimation}, representation learning methods~\cite{johansson2016learning,shalit2017estimating,shi2019adapting,wu2022learning,wang2023optimal}, and generative methods~\cite{louizos2017causal,yoon2018ganite}. Unlike the existing work devoted to estimating CATE at the intervention level for subgroups, our work focuses on counterfactual inference at the more challenging and fine-grained individual level.

{\bf Counterfactual Inference.} Counterfactual inference involves the identification and estimation of counterfactual outcomes.
  For identification,  \cite{Shpitser2007-UAI} provided an algorithm  leveraging counterfactual graphs to identify counterfactual queries. In addition, \cite{Correa2021-NIPS} discussed the identifiability of nested counterfactuals within a given causal graph. 
 More relevant to our work, 
   \cite{Xie-etal2023-attribution} and \cite{Lu-etal2020-attribution} studied the identifiability assumptions in the setting of backdoor criterion under homogeneity and
strict monotonicity assumptions. 
Several methods focus on determining its bounds with less stringent assumptions, such as \cite{Tian-Wu2025, Wu-etal-2024-Harm, Bakle-1994-UAI, Tian2000}. 
In addition, \citep{Wu-Mao2025} proposed a method for identifying the joint distribution of potential outcomes using multiple experimental datasets. 

 For estimation, \cite{Pearl-etal2016-primer}  introduced a three-step procedure for counterfactual inference. Many machine learning methods estimate counterfactual outcomes in this framework, such as  \cite{Mesnard2021-ICML, Brouwer2022, Lu-etal2020-attribution, Nasr-Esfahany-2023-attribution, Shah2022, Yan2023, Chao-etal2023}. 
 Recently, \cite{Xie-etal2023-attribution} employed quantile regression to estimate the counterfactual outcomes,  effectively circumventing the need for SCM estimation. 
 In our work, we extend the above methods in both identification and estimation.  
Recently, counterfactual inference methods have been extensively applied across various application scenarios,
   such as counterfactual fairness~\citep{kusner2017counterfactual, Zuo2023, Anthis2023-NIPS, Kavouras2023-NIPS, Chen2023-NIPS, Jin-etal2023}, policy evaluation and improvement~\citep{Wu-etal-2025-Safe, Wu-etal-2025-Compare, Wu-etal-ShortLong, yang2024learning}, reinforcement learning~\citep{Lu-etal2020-attribution, Tsirtsis2023-NIPS, Liu2023-NIPS, Shao2023-NIPS, Meulemans-etal-NIPS23, Haugh-ICML23, Zenati-ICML23}, imitation learning~\citep{Sun2023-NIPS, Zeng-etal2025-imitation}, counterfactual generation~\citep{Yan2023, Prabhu2023-NIPS, Feder2023-NIPS, Ribeiro-ICML23}, counterfactual explanation~\citep{Kenny2023-NIPS, Raman2023-NIPS, Hamman-ICML23, Wu-eatl-ICML23, Ley-ICML23}, counterfactual harm~\citep{Wu-etal-2024-Harm, Wu-etal-2025-Safe,2022counterfactualharm, li2023trustworthy}, physical audiovisual commonsense reasoning~\citep{Lv2023-NIPS}, interpretable time series prediction~\citep{Yan-etal-ICML23}, classification and detection in medical imaging~\citep{Fontanella-ICML23}, data valuation~\citep{Liu-ICML23}, etc.  
   Therefore, developing novel counterfactual inference methods holds significant practical implications.

\vspace{-2pt}
\section{Conclusion}

This work addresses the fundamental challenge of counterfactual inference in the absence of a known SCM and under heterogeneous endogenous variables. 
 We first introduce the rank preservation assumption to identify counterfactual outcomes, showing that it is slightly weaker than the homogeneity and monotonicity assumptions. 
 Then, we propose a novel ideal loss for unbiased learning of counterfactual outcomes and develop a kernel-based estimator for practical implementation. The convexity of the ideal loss and the unbiased nature of the proposed estimator contribute to the robustness and reliability of our method. 
 A potential limitation arises when the propensity score is extremely small in certain data sparsity scenarios, which may cause instability in the estimation method. Further investigation is warranted to address and overcome this challenge.  

\begin{ack}
This research was supported 
 the National Natural Science Foundation of China (No. 12301370),  
the BTBU Digital Business Platform Project by BMEC, and 
 the Beijing Key Laboratory of Applied Statistics and Digital Regulation.  
\end{ack}

\bibliographystyle{unsrt}
\bibliography{refs.bib}

\begin{thebibliography}{10}

\bibitem{Hernan-Robins2020}
M.A. Hern{\'a}n and J.~M. Robins.
\newblock {\em Causal Inference: What If}.
\newblock Boca Raton: Chapman and Hall/CRC, 2020.

\bibitem{Imbens-Rubin2015}
G.~W. Imbens and D.~B. Rubin.
\newblock {\em Causal Inference For Statistics Social and Biomedical Science}.
\newblock Cambridge University Press, 2015.

\bibitem{kallus2020-policy}
Nathan Kallus and Masatoshi Uehara.
\newblock Double reinforcement learning for efficient off-policy evaluation in
  markov decision processes.
\newblock {\em Journal of Machine Learning Research}, 21:1--63, 2020.

\bibitem{Pearl-Mackenzie2018}
Judea Pearl and Dana Mackenzie.
\newblock {\em The Book of Why: The New Science of Cause and Effect}.
\newblock Hachette Book Group, 2018.

\bibitem{Bareinboim-etal2022}
Elias Bareinboim, Juan~D. Correa, Duligur Ibeling, and Thomas Icard.
\newblock {\em On Pearl’s Hierarchy and the Foundations of Causal Inference}.
\newblock ACM, 2022.

\bibitem{Mesnard2021-ICML}
Thomas Mesnard, Théophane Weber, Fabio Viola, Shantanu Thakoor, Alaa Saade,
  Anna Harutyunyan, Will Dabney, Tom Stepleton, Nicolas Heess, Arthur Guez,
  Éric Moulines, Marcus Hutter, Lars Buesing, and Rémi Munos.
\newblock Counterfactual credit assignment in model-free reinforcement
  learning.
\newblock In {\em Proceedings of the 38th International Conference on Machine
  Learning}, page 7654–7664. PMLR, 2021.

\bibitem{Budhathoki-etal2022}
Kailash Budhathoki, Lenon Minorics, Patrick Bloebaum, and Dominik Janzing.
\newblock Causal structure-based root cause analysis of outliers.
\newblock In {\em International Conference on Machine Learning}. PMLR, 2022.

\bibitem{Pearl-etal2016-primer}
Judea Pearl, Madelyn Glymour, and Nicholas~P. Jewell.
\newblock {\em Causal Inference in Statistics: A Primer}.
\newblock John Wiley \& Sons, 2016.

\bibitem{Dawid2022}
A.~Philip Dawid and Monica Musio.
\newblock Effects of causes and causes of effects.
\newblock {\em Annual Review of Statistics and Its Application}, 9:261--287,
  2022.

\bibitem{Tian-Wu2025}
Zhaoqing Tian and Peng Wu.
\newblock Semiparametric efficient inference for the probability of necessary
  and sufficient causation.
\newblock {\em Statistics in medicine}, 44(18-19):e70242, 2025.

\bibitem{Wu-Mao2025}
Peng Wu and Xiaojie Mao.
\newblock The promises of multiple experiments: Identifying joint distribution
  of potential outcomes.
\newblock {\em arXiv preprint arXiv:2504.20470}, 2025.

\bibitem{imai2023principal}
Kosuke Imai and Zhichao Jiang.
\newblock Principal fairness for human and algorithmic decision-making.
\newblock {\em Statistical Science}, 38(2):317--328, 2023.

\bibitem{Wu-etal-2024-Harm}
Peng Wu, Peng Ding, Zhi Geng, and Yue Li.
\newblock Quantifying individual risk for binary outcome.
\newblock {\em arXiv preprint arXiv:2402.10537}, 2024.

\bibitem{Wu-etal-2025-Safe}
Peng Wu, Qing Jiang, and Shanshan Luo.
\newblock Safe individualized treatment rules with controllable harm rates.
\newblock {\em arXiv preprint arXiv:2505.05308}, 2025.

\bibitem{Ibeling-2020-AAAI}
Duligur Ibeling and Thomas Icard.
\newblock Probabilistic reasoning across the causal hierarchy.
\newblock In {\em Proceedings of the Thirty-Fourth AAAI Conference on
  Artificial Intelligence}, 2020.

\bibitem{Brouwer2022}
Edward~De Brouwer.
\newblock Deep counterfactual estimation with categorical background variables.
\newblock {\em Advances in Neural Information Processing Systems}, 2022.

\bibitem{Xie-etal2023-attribution}
Shaoan Xie, Biwei Huang, Bin Gu, Tongliang Liu, and Kun Zhang.
\newblock Advancing counterfactual inference through quantile regression.
\newblock In {\em ICML Workshop on Counterfactuals in Minds and Machines},
  2023.

\bibitem{shimizu2006linear}
Shohei Shimizu, Patrik~O Hoyer, Aapo Hyv{\"a}rinen, Antti Kerminen, and Michael
  Jordan.
\newblock A linear non-gaussian acyclic model for causal discovery.
\newblock {\em Journal of Machine Learning Research}, 7(10), 2006.

\bibitem{hoyer2008nonlinear}
Patrik Hoyer, Dominik Janzing, Joris~M Mooij, Jonas Peters, and Bernhard
  Sch{\"o}lkopf.
\newblock Nonlinear causal discovery with additive noise models.
\newblock {\em Advances in Neural Information Processing Systems}, 21, 2008.

\bibitem{peters2014causal}
Jonas Peters, Joris~M Mooij, Dominik Janzing, and Bernhard Sch{\"o}lkopf.
\newblock Causal discovery with continuous additive noise models.
\newblock 2014.

\bibitem{Lu-etal2020-attribution}
Chaochao Lu, Biwei Huang, Ke~Wang, Jo´se~Miguel Hernández-Lobato, Kun Zhang,
  and Bernhard Schölkopf.
\newblock Sample-efficient reinforcement learning viacounterfactual-based data
  augmentation.
\newblock In {\em Offline Reinforcement Learning Workshop at Neural Information
  Processing Systems}, 2020.

\bibitem{Nasr-Esfahany-2023-attribution}
Arash Nasr-Esfahany, Mohammad Alizadehand, and Devavrat Shah.
\newblock Counterfactual identifiability of bijective causal models.
\newblock In {\em International Conference on Machine Learning}. PMLR, 2023.

\bibitem{pearl2009causality}
Judea Pearl.
\newblock {\em Causality}.
\newblock Cambridge university press, 2009.

\bibitem{Holland1986}
Paul~W. Holland.
\newblock Statistics and causal inference.
\newblock {\em Journal of the American Statistical Association}, 81:945--960,
  1986.

\bibitem{Morgan-Winship-2015}
Stephen~L. Morgan and Christopher Winship.
\newblock {\em Counterfactuals and Causal Inference: Methods and Principles for
  Social Research}.
\newblock Cambridge University Press, second edition, 2015.

\bibitem{Albert-etal2005}
Jeffrey~M. Albert, Gary~L. Gadbury, and Edward~J. Mascha.
\newblock Assessing treatment effect heterogeneity in clinical trialswith
  blocked binary outcomes.
\newblock {\em Biometrical Journal}, 47:662--673, 2005.

\bibitem{Heckman1997}
James~J. Heckman, Jeffrey Smith, and Nancy Clements.
\newblock Making the most out of programme evaluations and social experiments:
  Accounting for heterogeneity in programme impacts.
\newblock {\em The Review of Economic Studies}, 64:487--535, 1997.

\bibitem{Djebbari-Smith2008}
Habiba Djebbari and Jeffrey~A. Smith.
\newblock Heterogeneous impacts in progresa.
\newblock {\em Journal of Econometrics}, 145:64--80, 2008.

\bibitem{Ding-etal2019}
Peng Ding, Avi Feller, and Luke Miratrix.
\newblock Decomposing treatment effect variation.
\newblock {\em Journal of the American Statistical Association}, 114:304--317,
  2019.

\bibitem{Lei-Candes2021}
Lihua Lei and Emmanuel~J. Cand{\`e}s.
\newblock Conformal inference of counterfactuals and individual treatment
  effects.
\newblock {\em Journal of the Royal Statistical Society: Series B (Statistical
  Methodology)}, 83:911--938, 2021.

\bibitem{Eli-etal2022}
Eli Ben-Michael, Kosuke Imai, and Zhichao Jiang.
\newblock Policy learning with asymmetric utilities.
\newblock {\em arXiv:2206.10479}, 2022.

\bibitem{Chernozhukov2005}
Victor Chernozhukov and Christian Hansen.
\newblock An {I}{V} model of quantile treatment effects.
\newblock {\em Econometrica}, 73:245--261, 2005.

\bibitem{Kendall1938}
Maurice~George Kendall.
\newblock A new measure of rank correlation.
\newblock {\em Biometrika}, 30:81--93, 1938.

\bibitem{Kendall1945}
Maurice~George Kendall.
\newblock The treatment of ties in ranking problem.
\newblock {\em Biometrika}, 33:239–251, 1945.

\bibitem{Koenker1978}
Roger Koenker and Gilbert Bassett.
\newblock Regression quantiles.
\newblock {\em Econometrica}, 46:33--50, 1978.

\bibitem{Firpo2007}
Sergio Firpo.
\newblock Efficient semiparametric estimation of quantile treatment effects.
\newblock {\em Econometrica}, 75:259--276, 2007.

\bibitem{Fan-Gijbels1996}
Jianqing Fan and Irene Gijbels.
\newblock {\em Local Polynomial Modelling and Its Applications}.
\newblock Chapman and Hall/CRC, 1996.

\bibitem{Li-Racine-2007}
Qi~Li and Jeff~S. Racine.
\newblock {\em Nonparametric econometrics}.
\newblock Princeton University Press, 2007.

\bibitem{Li-etal2024-interference}
Haoxuan Li, Chunyuan Zheng, Sihao Ding, Peng Wu, Zhi Geng, Fuli Feng, and
  Xiangnan He.
\newblock Be aware of the neighborhood effect: Modeling selection bias under
  interference.
\newblock In {\em International conference on learning representations}, 2024.

\bibitem{kunzel2019metalearners}
S{\"o}ren~R K{\"u}nzel, Jasjeet~S Sekhon, Peter~J Bickel, and Bin Yu.
\newblock Metalearners for estimating heterogeneous treatment effects using
  machine learning.
\newblock {\em Proceedings of the national academy of sciences},
  116(10):4156--4165, 2019.

\bibitem{johansson2016learning}
Fredrik Johansson, Uri Shalit, and David Sontag.
\newblock Learning representations for counterfactual inference.
\newblock In {\em International Conference on Machine Learning}, pages
  3020--3029. PMLR, 2016.

\bibitem{shalit2017estimating}
Uri Shalit, Fredrik~D Johansson, and David Sontag.
\newblock Estimating individual treatment effect: generalization bounds and
  algorithms.
\newblock In {\em International Conference on Machine Learning}, pages
  3076--3085. PMLR, 2017.

\bibitem{louizos2017causal}
Christos Louizos, Uri Shalit, Joris~M Mooij, David Sontag, Richard Zemel, and
  Max Welling.
\newblock Causal effect inference with deep latent-variable models.
\newblock {\em Advances in Neural Information Processing Systems}, 30, 2017.

\bibitem{shi2019adapting}
Claudia Shi, David Blei, and Victor Veitch.
\newblock Adapting neural networks for the estimation of treatment effects.
\newblock {\em Advances in Neural Information Processing Systems}, 32, 2019.

\bibitem{wu2022learning}
Anpeng Wu, Junkun Yuan, Kun Kuang, Bo~Li, Runze Wu, Qiang Zhu, Yueting Zhuang,
  and Fei Wu.
\newblock Learning decomposed representations for treatment effect estimation.
\newblock {\em IEEE Transactions on Knowledge and Data Engineering},
  35(5):4989--5001, 2022.

\bibitem{zhong2022descn}
Kailiang Zhong, Fengtong Xiao, Yan Ren, Yaorong Liang, Wenqing Yao, Xiaofeng
  Yang, and Ling Cen.
\newblock Descn: Deep entire space cross networks for individual treatment
  effect estimation.
\newblock In {\em Proceedings of the 28th ACM SIGKDD Conference on Knowledge
  Discovery and Data Mining}, pages 4612--4620, 2022.

\bibitem{wang2023optimal}
Hao Wang, Jiajun Fan, Zhichao Chen, Haoxuan Li, Weiming Liu, Tianqiao Liu,
  Quanyu Dai, Yichao Wang, Zhenhua Dong, and Ruiming Tang.
\newblock Optimal transport for treatment effect estimation.
\newblock {\em Advances in Neural Information Processing Systems}, 2023.

\bibitem{yoon2018ganite}
Jinsung Yoon, James Jordon, and Mihaela Van Der~Schaar.
\newblock Ganite: Estimation of individualized treatment effects using
  generative adversarial nets.
\newblock In {\em International conference on learning representations}, 2018.

\bibitem{yao2018representation}
Liuyi Yao, Sheng Li, Yaliang Li, Mengdi Huai, Jing Gao, and Aidong Zhang.
\newblock Representation learning for treatment effect estimation from
  observational data.
\newblock {\em Advances in Neural Information Processing Systems}, 31, 2018.

\bibitem{hill2011bayesian}
Jennifer~L Hill.
\newblock Bayesian nonparametric modeling for causal inference.
\newblock {\em Journal of Computational and Graphical Statistics},
  20(1):217--240, 2011.

\bibitem{lalonde1986evaluating}
Robert~J LaLonde.
\newblock Evaluating the econometric evaluations of training programs with
  experimental data.
\newblock {\em The American economic review}, pages 604--620, 1986.

\bibitem{Kosorok+Laber:2019}
Michael~R. Kosorok and Eric~B. Laber.
\newblock Precision medicine.
\newblock {\em Annual Review of Statistics and Its Application}, 6:263--86,
  2019.

\bibitem{Dudik2011-ICML}
Miroslav Dudík, John Langford, and Lihong Li.
\newblock Doubly robust policy evaluation and learning.
\newblock In {\em International Conference on Machine Learning}, page
  1097–1104. PMLR, 2011.

\bibitem{rosenbaum1983central}
Paul~R Rosenbaum and Donald~B Rubin.
\newblock The central role of the propensity score in observational studies for
  causal effects.
\newblock {\em Biometrika}, 70(1):41--55, 1983.

\bibitem{schwab2018perfect}
Patrick Schwab, Lorenz Linhardt, and Walter Karlen.
\newblock Perfect match: A simple method for learning representations for
  counterfactual inference with neural networks.
\newblock {\em arXiv preprint arXiv:1810.00656}, 2018.

\bibitem{chipman2010bart}
Hugh~A Chipman, Edward~I George, and Robert~E McCulloch.
\newblock Bart: Bayesian additive regression trees.
\newblock {\em The Annals of Applied Statistics}, 4(1):266--298, 2010.

\bibitem{wager2018estimation}
Stefan Wager and Susan Athey.
\newblock Estimation and inference of heterogeneous treatment effects using
  random forests.
\newblock {\em Journal of the American Statistical Association},
  113(523):1228--1242, 2018.

\bibitem{Shpitser2007-UAI}
Ilya Shpitser and Judea Pearl.
\newblock What counterfactuals can be tested.
\newblock In {\em Proceedings of the Twenty-Third Conference on Uncertainty in
  Artificial Intelligence}, 2007.

\bibitem{Correa2021-NIPS}
Juan~D. Correa, Sanghack Lee, and Elias Bareinboim.
\newblock Nested counterfactual identification from arbitrary surrogate
  experiments.
\newblock {\em Advances in Neural Information Processing Systems}, 2021.

\bibitem{Bakle-1994-UAI}
Alexander Balke and Judea Pearl.
\newblock Counterfactual probabilities: computational methods, bounds and
  applications.
\newblock In {\em Proceedings of the Tenth international conference on
  Uncertainty in artificial intelligence}, 1994.

\bibitem{Tian2000}
Jin Tian and Judea Pearl.
\newblock Probabilities of causation: Bounds and identification.
\newblock {\em Annals of Mathematics and Artificial Intelligence},
  28:287–313, 2000.

\bibitem{Shah2022}
Abhin Shah, Raaz Dwivedi, Devavrat Shah, and Gregory~W. Wornell.
\newblock On counterfactual inference with unobserved confounding.
\newblock In {\em NeurIPS 2022 Workshop on Causality for Real-world Impact},
  2022.

\bibitem{Yan2023}
Hanqi Yan, Lingjing Kong, Lin Gui, Yuejie Chi, Eric Xing, Yulan He, and Kun
  Zhang.
\newblock Counterfactual generation with identifiability guarantee.
\newblock {\em Advances in Neural Information Processing Systems}, 2023.

\bibitem{Chao-etal2023}
Patrick Chao, Patrick Blöbaum, and Shiva~Prasad Kasiviswanathan.
\newblock Interventional and counterfactual inference with diffusion models.
\newblock {\em arXiv:2302.00860}, 2023.

\bibitem{kusner2017counterfactual}
Matt~J Kusner, Joshua Loftus, Chris Russell, and Ricardo Silva.
\newblock Counterfactual fairness.
\newblock {\em Advances in Neural Information Processing Systems}, 30, 2017.

\bibitem{Zuo2023}
Zhiqun Zuo, Mohammad~Mahdi Khalili, and Xueru Zhang.
\newblock Counterfactually fair representation.
\newblock {\em Advances in Neural Information Processing Systems}, 2023.

\bibitem{Anthis2023-NIPS}
Jacy~Reese Anthis and Victor Veitch.
\newblock Causal context connects counterfactual fairness to robust prediction
  and group fairness.
\newblock {\em Advances in Neural Information Processing Systems}, 2023.

\bibitem{Kavouras2023-NIPS}
Loukas Kavouras, Konstantinos Tsopelas, Giorgos Giannopoulos, Dimitris
  Sacharidis, Eleni Psaroudaki, Nikolaos Theologitis, Dimitrios Rontogiannis,
  Dimitris Fotakis, and Ioannis Emiris.
\newblock Fairness aware counterfactuals for subgroups.
\newblock {\em Advances in Neural Information Processing Systems}, 2023.

\bibitem{Chen2023-NIPS}
Ruizhe Chen, Jianfei Yang, Huimin Xiong, Jianhong Bai, Tianxiang Hu, Jin Hao,
  YANG FENG, Joey~Tianyi Zhou, Jian Wu, and Zuozhu Liu.
\newblock Fast model debias with machine unlearning.
\newblock {\em Advances in Neural Information Processing Systems}, 2023.

\bibitem{Jin-etal2023}
Jinqiu Jin, Haoxuan Li, Fuli Feng, Sihao Ding, Peng Wu, and Xiangnan He.
\newblock Fairly recommending with social attributes: a flexible and
  controllable optimization approach.
\newblock {\em Advances in Neural Information Processing Systems}, 2023.

\bibitem{Wu-etal-2025-Compare}
Peng Wu, Shanshan Luo, and Zhi Geng.
\newblock On the comparative analysis of average treatment effects estimation
  via data combination.
\newblock {\em Journal of the American Statistical Association}, 2025.

\bibitem{Wu-etal-ShortLong}
Peng Wu, Ziyu Shen, Feng Xie, Zhongyao Wang, Chunchen Liu, and Yan Zeng.
\newblock Policy learning for balancing short-term and long-term rewards.
\newblock In {\em International Conference on Machine Learning}, 2024.

\bibitem{yang2024learning}
Qinwei Yang, Xueqing Liu, Yan Zeng, Ruocheng Guo, Yang Liu, and Peng Wu.
\newblock Learning the optimal policy for balancing short-term and long-term
  rewards.
\newblock {\em Advances in Neural Information Processing Systems}, 2024.

\bibitem{Tsirtsis2023-NIPS}
Stratis Tsirtsis and Manuel~Gomez Rodriguez.
\newblock Finding counterfactually optimal action sequences in continuous state
  spaces.
\newblock {\em Advances in Neural Information Processing Systems}, 2023.

\bibitem{Liu2023-NIPS}
Yao Liu, Pratik Chaudhari, and Rasool Fakoor.
\newblock Budgeting counterfactual for offline rl.
\newblock {\em Advances in Neural Information Processing Systems}, 2023.

\bibitem{Shao2023-NIPS}
Jianzhun Shao, Yun Qu, Chen Chen, Hongchang Zhang, and Xiangyang Ji.
\newblock Counterfactual conservative q learning for offline multi-agent
  reinforcement learning.
\newblock {\em Advances in Neural Information Processing Systems}, 2023.

\bibitem{Meulemans-etal-NIPS23}
Alexander Meulemans, Simon Schug, Seijin Kobayashi, and Greg~Wayne
  Nathaniel~Daw.
\newblock Would i have gotten that reward? long-term credit assignment by
  counterfactual contribution analysis.
\newblock {\em Advances in Neural Information Processing Systems}, 2023.

\bibitem{Haugh-ICML23}
Martin~B Haugh and Raghav Singal.
\newblock Counterfactual analysis in dynamic latent state models.
\newblock {\em International Conference on Machine Learning}, page
  12647–12677, 2023.

\bibitem{Zenati-ICML23}
Houssam Zenati, Eustache Diemert, Matthieu Martin, Julien Mairal, and Pierre
  Gaillard.
\newblock Sequential counterfactual risk minimization.
\newblock {\em International Conference on Machine Learning}, page
  40681–40706, 2023.

\bibitem{Sun2023-NIPS}
Zexu Sun, Bowei He, Jinxin Liu, Xu~Chen, Chen Ma, and Shuai Zhang.
\newblock Offline imitation learning with variational counterfactual reasoning.
\newblock {\em Advances in Neural Information Processing Systems}, 2023.

\bibitem{Zeng-etal2025-imitation}
Yan Zeng, Shenglan Nie, Feng Xie, Libo Huang, Peng Wu, and Zhi Geng.
\newblock Confounded causal imitation learning with instrumental variables.
\newblock {\em arXiv preprint arXiv:2507.17309}, 2025.

\bibitem{Prabhu2023-NIPS}
Viraj~Uday Prabhu, Sriram Yenamandra, Prithvijit Chattopadhyay, and Judy
  Hoffman.
\newblock Lance: Stress-testing visual models by generating language-guided
  counterfactual images.
\newblock {\em Advances in Neural Information Processing Systems}, 2023.

\bibitem{Feder2023-NIPS}
Amir Feder, Yoav Wald, Claudia Shi, Suchi Saria, and David Blei.
\newblock Data augmentations for improved (large) language model
  generalization.
\newblock {\em Advances in Neural Information Processing Systems}, 2023.

\bibitem{Ribeiro-ICML23}
Fabio De~Sousa Ribeiro, Tian Xia, Miguel Monteiro, Nick Pawlowski, and Ben
  Glocker.
\newblock High fidelity image counterfactuals with probabilistic causal model.
\newblock {\em International Conference on Machine Learning}, page 7390–7425,
  2023.

\bibitem{Kenny2023-NIPS}
Eoin~M. Kenny and Weipeng~Fuzzy Huang.
\newblock The utility of “even if” semifactual explanation to optimise
  positive outcomes.
\newblock {\em Advances in Neural Information Processing Systems}, 2023.

\bibitem{Raman2023-NIPS}
Chirag Raman, Alec Nonnemaker, Amelia Villegas-Morcillo, Hayley Hung, and Marco
  Loog.
\newblock Why did this model forecast this future? information-theoretic
  saliency for counterfactual explanations of probabilistic regression models.
\newblock {\em Advances in Neural Information Processing Systems}, 2023.

\bibitem{Hamman-ICML23}
Faisal Hamman, Erfaun Noorani, Saumitra Mishra, Daniele Magazzeni, and
  Sanghamitra Dutta.
\newblock Robust counterfactual explanations for neural networks with
  probabilistic guarantees.
\newblock {\em International Conference on Machine Learning}, page
  12351–12367, 2023.

\bibitem{Wu-eatl-ICML23}
Zhengxuan Wu, Karel D'Oosterlinck, Atticus Geiger, Amir Zur, and Christopher
  Potts.
\newblock Causal proxy models for concept-based model explanations.
\newblock {\em International Conference on Machine Learning}, page
  37313–37334, 2023.

\bibitem{Ley-ICML23}
Dan Ley, Saumitra Mishra, and Daniele Magazzeni.
\newblock {GLOBE}-{CE}: A translation based approach for global counterfactual
  explanations.
\newblock page 19315–19342. PMLR, 2023.

\bibitem{2022counterfactualharm}
Jonathan~G Richens, Rory Beard, and Daniel~H Thompson.
\newblock Counterfactual harm.
\newblock {\em Advances in Neural Information Processing Systems}, 2022.

\bibitem{li2023trustworthy}
Haoxuan Li, Chunyuan Zheng, Yixiao Cao, Zhi Geng, Yue Liu, and Peng Wu.
\newblock Trustworthy policy learning under the counterfactual no-harm
  criterion.
\newblock In {\em International Conference on Machine Learning}, pages
  20575--20598. PMLR, 2023.

\bibitem{Lv2023-NIPS}
Changsheng Lv, Shuai Zhang, Yapeng Tian, Mengshi Qi, and Huadong Ma.
\newblock Disentangled counterfactual learning for physical audiovisual
  commonsense reasoning.
\newblock {\em Advances in Neural Information Processing Systems}, 2023.

\bibitem{Yan-etal-ICML23}
Jingquan Yan and Hao Wang.
\newblock Self-interpretable time series prediction with counterfactual
  explanations.
\newblock {\em International Conference on Machine Learning}, page
  39110–39125, 2023.

\bibitem{Fontanella-ICML23}
Alessandro Fontanella, Antreas Antoniou, Wenwen Li, Joanna Wardlaw, Grant Mair,
  Emanuele Trucco, and Amos Storkey.
\newblock Acat: Adversarial counterfactual attention for classification and
  detection in medical imaging.
\newblock {\em International Conference on Machine Learning}, page
  10153–10169, 2023.

\bibitem{Liu-ICML23}
Zhihong Liu, Hoang Anh, Xiangyu Chang, Xi~Chen, and Ruoxi Jia.
\newblock 2d-shapley: A framework for fragmented data valuation.
\newblock {\em International Conference on Machine Learning}, page
  21730–21755, 2023.

\bibitem{vdv-1996}
Aad~W. van~der Vaart and Jon~A. Wellner.
\newblock {\em Weak convergence and empirical processes: with application to
  statistics}.
\newblock Springer, 1996.

\end{thebibliography}


\newpage
\section*{NeurIPS Paper Checklist}

\begin{enumerate}
\item {\bf Claims}
    \item[] Question: Do the main claims made in the abstract and introduction accurately reflect the paper's contributions and scope?
    \item[] Answer:
    \answerYes{} 
    \item[] Justification: We provide main claims and contributions in abstract and introduction.

\item {\bf Limitations}
    \item[] Question: Does the paper discuss the limitations of the work performed by the authors?
    \item[] Answer:  \answerYes{} 
    \item[] Justification: We discuss the limitations of the work in Conclusion. 

\item {\bf Theory assumptions and proofs}
    \item[] Question: For each theoretical result, does the paper provide the full set of assumptions and a complete (and correct) proof?
    \item[] Answer: \answerYes{} 
    \item[] Justification: We provide the full set of assumptions in Sections \ref{sec5} and \ref{sec6}, and the complete proofs in Appendices \ref{app-a} and \ref{app-b}.

    \item {\bf Experimental result reproducibility}
    \item[] Question: Does the paper fully disclose all the information needed to reproduce the main experimental results of the paper to the extent that it affects the main claims and/or conclusions of the paper (regardless of whether the code and data are provided or not)?
    \item[] Answer: \answerYes{}
    \item[] Justification: We provide a detailed description of the experimental process in Section \ref{sec7} and  Appendix \ref{app-e}.

\item {\bf Open access to data and code}
    \item[] Question: Does the paper provide open access to the data and code, with sufficient instructions to faithfully reproduce the main experimental results, as described in supplemental material?
    \item[] Answer: \answerYes{} 
    \item[] Justification: We share the data and code in the supplementary material. 

\item {\bf Experimental setting/details}
    \item[] Question: Does the paper specify all the training and test details (e.g., data splits, hyperparameters, how they were chosen, type of optimizer, etc.) necessary to understand the results?
    \item[] Answer: \answerYes{} 
    \item[] Justification: We provide the experimental setting and details. See details in Section \ref{sec7} and Appendix \ref{app-e}. 

\item {\bf Experiment statistical significance}
    \item[] Question: Does the paper report error bars suitably and correctly defined or other appropriate information about the statistical significance of the experiments?
    \item[] Answer: \answerYes{}
    \item[] Justification: We report error bars and standard deviations in the main comparative experiments.

\item {\bf Experiments compute resources}
    \item[] Question: For each experiment, does the paper provide sufficient information on the computer resources (type of compute workers, memory, time of execution) needed to reproduce the experiments?
    \item[] Answer: \answerYes{} 
    \item[] Justification: We provide sufficient information on the computer resources in Appendix \ref{app-e}.

\item {\bf Code of ethics}
    \item[] Question: Does the research conducted in the paper conform, in every respect, with the NeurIPS Code of Ethics \url{https://neurips.cc/public/EthicsGuidelines}?
    \item[] Answer: \answerYes{} 
    \item[] Justification: We have reviewed the NeurIPS Code of Ethics and our paper conforms with it. 

\item {\bf Broader impacts}
    \item[] Question: Does the paper discuss both potential positive societal impacts and negative societal impacts of the work performed?
    \item[] Answer: \answerNA{} 
    \item[] Justification: There is no societal impact of the work performed. 
    
\item {\bf Safeguards}
    \item[] Question: Does the paper describe safeguards that have been put in place for responsible release of data or models that have a high risk for misuse (e.g., pretrained language models, image generators, or scraped datasets)?
    \item[] Answer: \answerNA{} 
    \item[] Justification: Our research does not have such risks.

\item {\bf Licenses for existing assets}
    \item[] Question: Are the creators or original owners of assets (e.g., code, data, models), used in the paper, properly credited and are the license and terms of use explicitly mentioned and properly respected?
    \item[] Answer: \answerYes{} 
    \item[] Justification: The assets used have been properly noted and credited. 

\item {\bf New assets}
    \item[] Question: Are new assets introduced in the paper well documented and is the documentation provided alongside the assets?
    \item[] Answer: \answerNA{} 
    \item[] Justification: No new assets are being released.

\item {\bf Crowdsourcing and research with human subjects}
    \item[] Question: For crowdsourcing experiments and research with human subjects, does the paper include the full text of instructions given to participants and screenshots, if applicable, as well as details about compensation (if any)? 
    \item[] Answer: \answerNA{}
    \item[] Justification: We do not have any studies or results regarding crowdsourcing experiments and human subjects.

\item {\bf Institutional review board (IRB) approvals or equivalent for research with human subjects}
    \item[] Question: Does the paper describe potential risks incurred by study participants, whether such risks were disclosed to the subjects, and whether Institutional Review Board (IRB) approvals (or an equivalent approval/review based on the requirements of your country or institution) were obtained?
    \item[] Answer: \answerNA{}
    \item[] Justification: We do not have any studies or results including study participants. 

\item {\bf Declaration of LLM usage}
    \item[] Question: Does the paper describe the usage of LLMs if it is an important, original, or non-standard component of the core methods in this research? Note that if the LLM is used only for writing, editing, or formatting purposes and does not impact the core methodology, scientific rigorousness, or originality of the research, declaration is not required.
    \item[] Answer: \answerNA{}. 
    \item[] Justification: The core methodological development in this research does not involve large language models (LLMs) as essential, original, or non-standard components. 

\end{enumerate}


\appendix

\newpage 


\section{Proofs in Sections \ref{sec4} and \ref{sec5}}  \label{app-a}

One can show Lemma \ref{lem1} by a similar argument of the proof of  Theorem 1 in \cite{Xie-etal2023-attribution}. 
For the sake of self-containedness, we provide a novel proof of it.   
 
{\bf Lemma \ref{lem1}} \emph{Under Assumptions \ref{assump1}-\ref{assump2}, $y_{x'}$is identifiable.} 
\begin{proof}[Proof of Lemma \ref{lem1}]  

First, the distributions $\P(Y_x| Z=z) $ and $\P(Y_{x'}| Z=z)$ can be identified as $\P(Y | X=x, Z=z)$ and $ \P(Y | X=x', Z=z)$, respectively, by the backdoor criterion (i.e., $(Y_x, Y_{x'}) \indep X | Z$) of the setting.  

Then, according to the model (\ref{eq1}), we can equivalently write 
	 \[ Y_x = f_Y(x, z, U_{x}), ~  Y_{x'} = f_Y(x', z, U_{x'}), \]
and $Y$ and $U_X$ in model (\ref{eq1}) can be expressed as $Y = \sum_{x\in \mathcal{X}}  \mathbb{I}(X =x) \cdot Y_x$ and $U_X = \sum_{x\in \mathcal{X}}  \mathbb{I}(X =x) \cdot U_{x}$, where $\mathcal{X}$ is the support set of $X$ and $\mathbb{I}(\cdot)$ is an indicator function.   
 Assumption \ref{assump1} implies that $U_X = U_{x} = U_{x'}$ conditional on $Z$, i.e., $Y_x = f_Y(x, z, U_{X}), ~  Y_{x'} = f_Y(x', z, U_{X})$. 

Finally, for the individual with observation $(X =x, Z=z, Y= y)$, we denote $(y_x, y_{x'})$ as the true values of  ($Y_x, Y_{x'}$) for this individual. 
 For this individual, we can identify the quantile of $y_x$ in the distribution of $\P(Y_x | Z=z) = \P(Y | X=x, Z=z)$, denoted by $\tau^*$. Let $u_{\tau^*}$ be the true value of $U_{X}$ for this individual, it is the $\tau^*$-quantile in the distribution $\P(U_X|Z=z)$, then we have 
      \begin{align*}
       \tau^* ={}& \P( Y_x \leq y_x  | Z=z)     \qquad \qquad \qquad \text{(by the definition of $\tau$)} \\
	     ={}& \P(  U_{x} \leq  u_{\tau}   | Z=z )  \qquad \qquad \quad  ~ \text{(by Assumption \ref{assump2})} \\
	     ={}& \P(  U_{x'} \leq  u_{\tau}   | Z=z )  \qquad \qquad ~~~   \text{(by Assumption \ref{assump1})} \\
	     ={}& \P(  Y_{x'} \leq f_Y(x', z, u_{\tau^*})  | Z=z )  \quad ~ \text{(by Assumption \ref{assump2})} \\  
	     ={}& \P(  Y_{x'} \leq y_{x'}  | Z=z )    \qquad \qquad \quad  \text{(by the definition of $y_{x'}$)},
      \end{align*} 	 
  which implies that  for this individual,  its rankings of $y_{x}$ and $y_{x'}$ are the same in the distributions of $\P(Y_x | Z=z)$ and $\P( Y_{x'} | Z=z )$, resepcctively.  Thus, $y_{x'}$is identified as the $\tau^*$-quantile of the distribution   $\P( Y_{x'} | Z=z ) = \P(Y| X=x', Z=z)$. 
  


\end{proof}

{\bf Proposition \ref{prop1}}  \emph{   Under Assumption \ref{assump3}, $y_{x'}$ is identified as 
   the $\tau^*$-th quantile of  
    $\P(Y| X=x', Z=z)$, where $\tau^*$ is the quantile of $y$ in the distribution of $\P(Y | X=x, Z=z)$.}   

\begin{proof}[Proof of Proposition \ref{prop1}]   
 For the individual with observation $(X =x, Z=z, Y= y)$, we denote $(y_x, y_{x'})$ as the true values of  ($Y_x, Y_{x'}$). Assumption \ref{assump3} implies that for this individual,  
 its rankings of $y_{x}$ and $y_{x'}$ are the same in the distributions of $\P(Y_x | Z=z)$ and $\P( Y_{x'} | Z=z )$,  
 respectively. Therefore, 
       \begin{equation} 
              \P( Y_x \leq y_x  | Z=z)  =  \P(  Y_{x'} \leq y_{x'}  | Z=z ).  
       \end{equation}

Since $y_x = y$ is observed and the distributions $\P(Y_x| Z=z) $ and $\P(Y_{x'}| Z=z)$ can be identified as $\P(Y | X=x, Z=z)$ and $ \P(Y | X=x', Z=z)$, respectively, by the backdoor criterion (i.e., $(Y_x, Y_{x'}) \indep X | Z$),  we can identify the quantile of $y_x$ in the distribution of $\P(Y | X=x, Z=z)$, denoted by $\tau^*$.  Then
	\[     \P(  Y_{x'} \leq y_{x'}  | Z=z ) = \tau^*,  \]
which yields that $\theta $ is identified as the $\tau^*$-quantile of  
   $\P(Y| X=x', Z=z)$. 

\end{proof}

The following Proposition \ref{prop2}$^*$ serves as a complement to Proposition \ref{prop2}. 

{\bf Proposition \ref{prop2}$^*$}  
Under Assumption \ref{assump1}, or more generally, if $U_x$ is a strictly monotone increasing function of $U_{x'}$, Assumption \ref{assump3} is equivalent to Assumption \ref{assump2}.

\begin{proof}[Proof of Proposition \ref{prop2}]  According to the model (\ref{eq1}), we can equivalently write 
  	 \[ Y_x = f_Y(x, z, U_{x}), ~  Y_{x'} = f_Y(x', z, U_{x'}). \]
Suppose that $U_x$ is a strictly monotone increasing function of $U_{x'}$ (Assumption \ref{assump1}, i.e., $U_x = U_{x'}$, is a special case of it). Under this condition, we next prove sufficiency and necessity, respectively. 

First, we show that Assumption \ref{assump2} implies Assumption \ref{assump3}. If Assumption \ref{assump2} holds, then $Y_x$ is a strictly monotonic function of $U_x$, and $Y_{x'}$ is a strictly monotonic function of $U_{x'}$. Since $U_x$ is a strictly monotone increasing function of $U_{x'}$, then $Y_x$ is a strictly increasing monotonic function of $Y_{x'}$, which leads to Assumption \ref{assump3}.  

Second, we show that  Assumption \ref{assump3} implies Assumption \ref{assump2}.  If Assumption \ref{assump3} holds, then given $Z=z$, 
   $Y_x$ is a strictly increasing function of $Y_{x'}$. 
When $U_x$ is a strictly monotone increasing function of $U_{x'}$ and note that    
    \[ Y_x = f_Y(x, z, U_{X}), ~  Y_{x'} = f_Y(x', z, U_{X}), \]
which implies that $f_Y$ is a strictly monotonic function of $U_X$, i.e.,  Assumption \ref{assump2} holds. 

 This finishes the proof. 
 
\end{proof}

   {\bf Proposition \ref{prop3}}  \emph{Under Assumption \ref{assump4}, the conclusion in Proposition \ref{prop1} also holds.}   

\begin{proof}[Proof of Proposition \ref{prop3}]  This can be shown through a proof analogous to that of Proposition \ref{prop1}.  

\end{proof} 


%

\section{Proofs in Section \ref{sec6}} \label{app-b} 

Recall that $l_{\tau}(\xi) = \tau \xi \cdot \mathbb{I}(\xi \geq 0)+ (\tau - 1)\xi \cdot \mathbb{I}(\xi < 0)$, and 
        \begin{align*}
        q(x, z; \tau) \triangleq{}& \inf_{y}\{y: \P( Y \leq y | X=x, Z=z) \geq \tau \} \\ 
              q_0(z; \tau)  \triangleq{}& \inf_{y}\{y: \P( Y_0 \leq y | Z=z) \geq \tau \} \\
                 q_1(z; \tau)   \triangleq{}& \inf_{y}\{y: \P( Y_1 \leq y | Z=z) \geq \tau \}. 
        \end{align*}

   {\bf Lemma \ref{prop4}}  We have that

\emph{(i) $q_x(Z; \tau) = \arg\min_{f} \bfE[ l_{\tau}(Y_x - f(Z))]$ for $x= 0 , 1$;} 

\emph{(ii)  $q(X, Z; \tau) = \arg\min_{f} \bfE[ l_{\tau}(Y - f(X, Z))]$.}


\begin{proof}[Proof of Lemma \ref{prop4}] We prove $q_x(Z; \tau) = \arg\min_{f} \bfE[ l_{\tau}(Y_x - f(Z))]$, and   
$q(X, Z; \tau) = \arg\min_{f} \bfE[ l_{\tau}(Y - f(X, Z))]$ can be derived by an exactly similar manner. We write 
\[ \bfE[ l_{\tau}(Y_x - f(Z))] = \bfE [ \bfE \{ l_{\tau}(Y_x - f(Z)) \mid Z\} ].   \]
To obtain the conclusion, note that $l_{\tau}(Y_x - f(Z))$ is always positive, it suffices to show that  
    \begin{equation} \label{eq-s5}
        q_x(z; \tau) = \arg \min_{f} \bfE [ l_{\tau}(Y_x - f(Z)) \mid Z = z ] 
    \end{equation}  
for any given $Z=z$. Next, we focus on analyzing the term $\bfE[ l_{\tau}(Y_x - f(Z)) \mid Z = z ]$. Given $Z=z$, $f(Z)$ is a constant and we denote it by $c$, then 
    \begin{align*}
         & \bfE[ l_{\tau}(Y_x - f(Z)) \mid Z = z ]  \\
       ={}&  \bfE[ l_{\tau}(Y_x - c) \mid Z = z ] \\
       ={}& \bfE\Big [  \tau (Y_x - c) \mathbb{I}(Y_x \geq c) + (\tau - 1) (Y_x - c) \mathbb{I}(Y_x < c)   \mid Z=z \Big ] \\
       ={}&  \tau \int_c^{\infty} (y_x - c) g(y_x|z) dy_x  + (\tau -1) \int_{-\infty}^c (y_x - c) g(y_x|z) dy_x,   
    \end{align*}
where $g(y_x|z)$ denotes the probability density function of $Y_x$ given $Z=z$.    

Since the check function is a convex function, differentiating $\bfE[ l_{\tau}(Y_x - c) \mid Z = z ]$ with respect to $c$ and setting the derivative to zero will yield the solution for the minimum
   \begin{align*}
       & \frac{\partial}{\partial c} \bfE[ l_{\tau}(Y_x - c) \mid Z = z ] \\
      ={}& \tau \int_c^{\infty} \frac{\partial}{\partial c}[(y_x - c) g(y_x|z)] dy_x + (\tau - 1) \int_{-\infty}^c \frac{\partial}{\partial c}[(y_x - c) g(y_x|z)] dy_x \\
      ={}& - \tau \Big (1 - \int_{-\infty}^{c} g(y_x|z)dy_x \Big ) + (1 - \tau) \int_{-\infty}^c g(y_x|z)dy_x.  
   \end{align*}
Then let $ \frac{\partial}{\partial c} \bfE[ l_{\tau}(Y_x - c) \mid Z = z ] = 0$ leads to that 
 \[      \int_{-\infty}^{c} g(y_x|z)dy_x = \tau,  \]
that is, $c = q_x(z; \tau)$. This completes the proof of Proposition \ref{prop4}. 

\end{proof}

{\bf Theorem \ref{thm-1}.} \emph{If the probability density function of $Y$ given $Z$ is continuous, then the loss  $R_{x'}(t; x, z, y)$ is minimized uniquely at $t^*$, where $t^*$ is the solution satisfying  
    \[   \P(Y_{x'} \leq t^*  |  Z=z ) = \P(Y_x \leq y  |  Z=z ).  \]}

\begin{proof}[Proof of Theorem \ref{thm-1}]
Recall that 
\begin{align*}
     R_{x'}(t| x, z, y)
   ={}&    \bfE \left [  | Y_{x'} - t |  ~ \Big | ~ Z=z \right ]   + \bfE \left [ \text{sign}(Y_x - y)  ~ \Big | ~  Z = z\right ] \cdot t. 
 \end{align*}
Let  $g(y_x|z)$ be the probability density function of $Y_x$ given $Z=z$.  By calculation, 
 \begin{align*}
       \bfE \left [  | Y_{x'} - t |  ~ \Big | ~ Z=z \right ] =   \int_{t}^{\infty} (y_{x'} - t) g(y_{x'}|z) dy_{x'} + \int_{-\infty}^{t} (t - y_{x'})  g(y_{x'}|z) dy_{x'}, 
 \end{align*} 
  \begin{align*}
       \frac{\partial}{\partial t}  \bfE \left [  | Y_{x'} - t |  ~ \Big | ~ Z=z \right ] 
       = -\Big( 1-  \int_{-\infty}^{t} g(y_{x'}|z) dy_{x'} \Big) + \int_{-\infty}^{t} g(y_{x'}|z) dy_{x'} 
       =  2 \P( Y_{x'} \leq t | Z = z) - 1,
  \end{align*}
and 
  \begin{align*}
         \bfE \left [  \text{sign}(Y_x - y)   ~ \Big | ~  Z = z\right ]
      = \bfE  \left [ -2 \mathbb{I}(Y_x \leq y) + 1  ~ \Big | ~  Z = z\right ] 
      = - 2 \P( Y_x \leq y | Z=z) + 1,  
  \end{align*}
we have   
    \begin{align*}
       \frac{\partial}{\partial t} R_{x'}(t| x, z, y) ={}& 2 \P(Y_{x'} \leq t | Z=z) - 1 +  \bfE \left [  \text{sign}(Y - y)   ~ \Big | ~  Z = z\right ] \\
       ={}& 2 \P(Y_{x'} \leq t | Z=z) - 1 - 2 \P(Y_x \leq y| Z=z ) +1 \\
       ={}& 2 \Big \{ \P(Y_{x'} \leq t | z)  -  \P(Y_{x} \leq y | z) \Big  \}.   
      \end{align*}
 Since 
    \begin{align*}
        \frac{\partial^2}{\partial t^2} R_{x'}(t| x, z, y) = 2 \partial \P(Y_{x'} \leq t | z) / \partial t = 2 g(y_{x'}=t|z)  \geq 0,
    \end{align*}
 $R_{x'}(t| x, z, y)$ is a convex function with respect to $t$. 
Letting $\frac{\partial}{\partial t} R_{x'}(t| x, z, y) = 0$ yields that 
     \[  \P(Y_{x'} \leq t | z)  -  \P(Y_{x} \leq y | z) = 0.     \]
That is, $R_{x'}(t| x, z, y)$ attains its minimum at $t = q_{x'}(z; \tau^*)$, where $\tau^*$ is the quantile of $y$ in the distribution $\P(Y_x| Z=z)$. 

\end{proof}

{\bf Proposition \ref{prop5}.} \emph{If $h \to 0$ as $N \to \infty$, $\hat p_x(z)$ and $\hat p_{x'}(z)$ are consistent estimates of $p_x(z)$ and $p_{x'}(z)$, and the density function of $Z$ is differentiable, then 
$$\hat R_{x'}(t; x, z, y) \xrightarrow{\P} R_{x'}(t; x, z, y),$$
where $\xrightarrow{\P}$ means convergence in probability.}

\begin{proof}[Proof of Proposition \ref{prop5}] 
For analyzing the theoretical properties of $\hat R_{x'}(t; x, z, y)$, we rewritten $\hat R_{x'}(t; x, z, y)$ as  
  \begin{align*}
        \hat R_{x'}(t; x, z, y) ={} \frac{ \sum_{k=1}^N K_h(Z_k -z) \frac{\mathbb{I}(X_k=x')}{\hat p_{x'}(Z_k)} | Y_k - t |  }{ \sum_{k=1}^N K_h(Z_k -z) } + \frac{ \sum_{k=1}^N K_h(Z_k -z) \frac{\mathbb{I}(X_k=x)}{\hat p_{x}(Z_k)} \cdot \text{sign}(Y_k - y) }{  \sum_{i=1}^N K_h(Z_k -z) }  \cdot t, 
  \end{align*}
where the capital letters denote random variables and lowercase letters denote their realizations. This is slightly different from that used in the main text. 

When $\hat p_x(z)$ and $\hat p_{x'}(z)$ are consistent estimates of $p_x(z)$ and $p_{x'}(z)$, to show the conclusion, it is sufficient to prove that 
  \begin{align}
      \frac{  \sum_{k=1}^N K_h(Z_k -z) \frac{\mathbb{I}(X_k=x')}{p_{x'}(Z_k)} | Y_k - t |  }{   \sum_{k=1}^N K_h(Z_k -z) } 
      \xrightarrow{\P}{}&    \bfE \left [ \frac{\mathbb{I}(X=x')}{p_{x'}(z)} | Y - t |  ~ \Big | ~ Z=z \right ] =  \bfE \left [  | Y_{x'} - t |  ~ \Big | ~ Z=z \right ],  \label{eq-s6}  \\
       \frac{  \sum_{k=1}^N K_h(Z_k -z) \frac{\mathbb{I}(X_k=x)}{p_{x}(Z_k)} \cdot \text{sign}(Y_k - y) }{  \sum_{i=1}^N K_h(Z_k -z) }    \xrightarrow{\P}{}&   \bfE \left [ \frac{\mathbb{I}(X=x)}{p_{x}(z)} \cdot \text{sign}(Y - y)  ~ \Big | ~  Z = z\right ]= \bfE \left [ \text{sign}(Y_x - y)  ~ \Big | ~  Z = z\right ].  \label{eq-s7}   
  \end{align}
  We prove equation (\ref{eq-s6}) only, as equation (\ref{eq-s7}) can be addressed similarly.

Note that 
\[      \frac{  \sum_{k=1}^N K_h(Z_k -z) \frac{\mathbb{I}(X_k=x')}{p_{x'}(Z_k)} | Y_k - t |  }{   \sum_{k=1}^N K_h(Z_k -z) }  =  \frac{ \frac 1 N \sum_{k=1}^N K_h(Z_k -z) \frac{\mathbb{I}(X_k=x')}{p_{x'}(Z_k)} | Y_k - t |  }{  \frac 1 N \sum_{k=1}^N K_h(Z_k -z) },   \]
we analyze the denominator and numerator on the right side of the equation separately. For the denominator, it is an average of $N$ independent random variables and converges to its expectation 
   $\bfE[ K_{h}(Z_k -z)]$
almost surely. Let  $g(z_k)$ be the probability density function of $Z_k$, and $g^{(1)}(z_k)$ is its first derivative.  
Since
     \begin{align}
         \bfE[ K_{h}(Z_k -z)] ={}& \int \frac{1}{h} K(\frac{z_k - z}{h}) g(z_k) dz_k  \notag \\
         ={}& \int K(u) g(z+hu)  du \qquad \text{(let $z_k=z+hu$)} \notag\\
         ={}& \int K(u)\cdot \{g(z) + g^{(1)}(z) h u + o(h) \} du \qquad \text{(by Taylor Expansion)} \notag \\
         ={}& g(z) \int K(u)du +  g^{(1)}(z) h \int K(u)u du + o(h) \notag\\
         ={}& g(z) + o(h)  \qquad \text{(by the definition of kernel function),}       \label{eq-s8}
     \end{align}
   when $h\to 0$ as $N \to \infty$,     
the denominator converges to $g(z)$ in probability.  

Next, we focus on dealing with the numerator, which also converges to its expectation.  
   \begin{align}
    & \bfE[ K_h(Z_k -z) \frac{\mathbb{I}(X_k=x')}{p_{x'}(Z_k)} | Y_k - t | ] \notag \\
   ={}& \bfE \Big [ K_h(Z_k -z) \bfE  \Big\{ \frac{\mathbb{I}(X_k=x')}{p_{x'}(Z_k)} | Y_{k} - t | \Big | Z_k   \Big\}   \Big ] \qquad \text{(by the law of iterated expectations)} \notag \\
   ={}&  \bfE \Big [ K_h(Z_k -z) \bfE  \Big\{ \frac{\mathbb{I}(X_k=x')}{p_{x'}(Z_k)} | Y_{x',k} - t | \Big | Z_k   \Big\}   \Big ] \qquad \text{(write $Y_k$ as the form of potential outcome)} \notag \\
   ={}&  \bfE \Big [ K_h(Z_k -z) \bfE  \Big\{ | Y_{x',k} - t | \Big | Z_k   \Big\}   \Big ]   \qquad \text{(by backdoor criterion $Y_{x',k}\indep X_k | Z_k$)}. \label{eq-s9}
   \end{align}
Define $m(Z) = \bfE[ |Y_{x'} -t| \big | Z]$ and $m^{(1)}(Z)$ is its first derivative, then the right side of equation (\ref{eq-s6}) is $m(z)$, and 
    \begin{align}
         \bfE \Big [ &  K_h(Z_k -z) \cdot \bfE  \Big\{ | Y_{x',k} - t | \Big | Z_k   \Big\}   \Big ] 
       ={} \bfE \Big [ K_h(Z_k -z) \cdot m(Z_k) \Big ] \notag\\
       ={}& \int \frac{1}{h} K(\frac{z_k-z}{h}) \cdot m(z_k) \cdot g(z_k) dz_k \notag \\
       ={}& \int K(u) \cdot m(z+hu) \cdot g(z+hu) du \qquad \text{(let $z_k = z+ hu$)}\notag \\
       ={}& \int K(u) \cdot \{m(z) + m^{(1)}(z) h u + o(h)   \} \cdot \{g(z) + g^{(1)}(z) h u + o(h)  \} du \qquad \text{(by Taylor Expansion)} \notag \\
       ={}&  m(z) g(z) + o(h). \label{eq-s10}
    \end{align}
Thus,  when $h\to 0$ as $N \to \infty$,  the numerator converges to $g(z)$ in probability.   

Combining equations (\ref{eq-s8}), (\ref{eq-s9}), and (\ref{eq-s10}) yields the equality (\ref{eq-s6}). This completes the proof. 

\end{proof}

{\bf Theorem \ref{thm5-4}} (Unbiasedness Preservation). \emph{ 
The loss $R_{x'}^{\mathrm{weight}}(t| x, z, y)$ is convex in terms of $t$ and is minimized uniquely at $t^*$, where  $t^*$ is the solution satisfying 
       $\P(Y_{x'} \leq t^* | Z=z) = \P(Y_x \leq y| Z=z).$}

\begin{proof}[Proof of Theorem \ref{thm5-4}]. 
It is sufficient to show that $R_{x'}^{\mathrm{weight}}(t| x, z, y)$ shares the same minimizer as $R_{x'}(t| x, z, y)$. By the backdoor criterion (i.e., $(Y_x, Y_{x'} \indep X | Z)$), we have that 
  \begin{align*}
        & R_{x'}^{\mathrm{weight}}(t| x, z, y) \\
       ={}& \bfE \left [ w(X, Z)   | Y_{x'} - t |  \big |  Z=z \right ]+\bfE \left [ w(X, Z)\text{sign}(Y_x - y)   \big |   Z = z\right ] \cdot t \\  
        ={}& \bfE[ w(X, Z) \mid Z=z ] \cdot  \bfE \left [   | Y_{x'} - t |  ~ \big | ~ Z=z \right ]  \\
        {}& + \bfE[ w(X, Z) \mid Z=z ] \cdot  \bfE \left [ w(X, Z) \cdot \text{sign}(Y_x - y)  ~ \big | ~  Z = z\right ] \cdot t  \\
        ={}&  \bfE[ w(X, Z) \mid Z=z ] \cdot R_{x'}(t| x, z, y).
  \end{align*}
Note that $\bfE[ w(X, Z) \mid Z=z ]$ is not related to $t$.  Consequently, $R_{x'}^{\mathrm{weight}}(t|x, z, y)$ shares the same unique minimizer as $R_{x'}(t|x, z, y)$. Both minimizers satisfy the condition 
       $\P(Y_{x'} \leq t^* | Z=z) = \P(Y_x \leq y| Z=z).$


\end{proof}

\medskip

{\bf Proposition \ref{prop5-5}} (Bias of the Estimated Loss). 
 \emph{If $h \to 0$ as $N \to \infty$,  $p_x(z)/\hat p_x(z)$ and $p_{x'}(z)/\hat p_{x'}(z)$ are differentiable with respect to $z$, and the density function of $Z$ is differentiable, then the bias of $\hat R_{x'}(t| x, z, y)$, defined by $\bfE[ \hat R_{x'}(t|x,z,y) ] - R_{x'}(t|x,z,y)$, is given as 
    \begin{align*}
        \text{Bias}(&\hat R_{x'})  ={}
     \delta_{p_{x'}}  \bfE \left [  | Y_{x'} - t |  ~ \big | ~ Z=z \right ]  + \delta_{p_{x}}  \bfE \left [ \text{sign}(Y_x - y)  ~ \big | ~  Z = z\right ] \cdot t + O(h^2), 
    \end{align*} 
where $\delta_{p_{x'}}  = (p_{x'}(z) - \hat p_{x'}(z))/\hat p_{x'}(z)$ and $\delta_{p_{x}}  = (p_{x}(z) - \hat p_{x}(z))/\hat p_{x}(z)$ are estimation errors of propensity scores. }

\begin{proof}[Proof of Proposition \ref{prop5-5}].  
This proof is similar to the proof of Proposition \ref{prop5}. Recall that 
  \begin{align*}
        \hat R_{x'}(t; x, z, y) ={} \frac{ \sum_{k=1}^N K_h(Z_k -z) \frac{\mathbb{I}(X_k=x')}{\hat p_{x'}(Z_k)} | Y_k - t |  }{ \sum_{k=1}^N K_h(Z_k -z) } + \frac{ \sum_{k=1}^N K_h(Z_k -z) \frac{\mathbb{I}(X_k=x)}{\hat p_{x}(Z_k)} \cdot \text{sign}(Y_k - y) }{  \sum_{i=1}^N K_h(Z_k -z) }  \cdot t, 
  \end{align*}
and 
 \[  R_{x'}(t| x, z, y) = \bfE \left [  | Y_{x'} - t |  ~ \big | ~ Z=z \right ]   + \bfE \left [ \text{sign}(Y_x - y)  ~ \big | ~  Z = z\right ] \cdot t, \]
Given the estimated propensity scores, to show the conclusion, it suffices to prove that 
  \begin{align}
     \bfE \left [  \frac{  \sum_{k=1}^N K_h(Z_k -z) \frac{\mathbb{I}(X_k=x')}{\hat p_{x'}(Z_k)} | Y_k - t |  }{   \sum_{k=1}^N K_h(Z_k -z) } \right ]  ={}&  
     \frac{p_{x'}(z)}{\hat p_{x'}(z)} \bfE \left [  | Y_{x'} - t |  ~ \big | ~ Z=z \right ] + O(h^2). \label{eq-s11}  \\
        \bfE \left [  \frac{  \sum_{k=1}^N K_h(Z_k -z) \frac{\mathbb{I}(X_k=x)}{p_{x}(Z_k)} \cdot \text{sign}(Y_k - y) }{  \sum_{i=1}^N K_h(Z_k -z) }  \right ]  ={}&   \frac{p_{x}(z)}{\hat p_{x}(z)} \bfE \left [ \text{sign}(Y_x - y)  ~ \big | ~  Z = z\right ] \cdot t + O(h^2). \label{eq-s12}   
  \end{align}
  We prove the equation (\ref{eq-s11}) only, as remaining equation (\ref{eq-s12}) can be addressed similarly.  
  Let  $g(z_k)$ be the probability density function of $Z_k$ and  $m(Z) = \bfE[ |Y_{x'} -t| \big | Z]$.   
By the proof of  Proposition \ref{prop5}, $N^{-1}\sum_{k=1}^N K_h(Z_k -z) = g(z) + O(h^2)$, and 
   \begin{align*}
   &  \bfE \left [  \frac 1 N \sum_{k=1}^N K_h(Z_k -z) \frac{\mathbb{I}(X_k=x')}{\hat p_{x'}(Z_k)} | Y_k - t |  \right ] \\
    & \bfE[ K_h(Z_k -z) \frac{\mathbb{I}(X_k=x')}{\hat p_{x'}(Z_k)} | Y_k - t | ] \notag \\
   ={}&  \bfE \Big [ K_h(Z_k -z) \cdot \frac{ p_{x'}(Z_k)}{\hat p_{x'}(Z_k)} \bfE  \Big\{ | Y_{x',k} - t | \Big | Z_k   \Big\}   \Big ]    \\
      ={}&  \bfE \Big [ K_h(Z -z) \cdot \frac{ p_{x'}(Z)}{\hat p_{x'}(Z)}\cdot m(Z) \Big ]    \\
       ={}&  \frac{ p_{x'}(z)}{\hat p_{x'}(z)} m(z) g(z) + O(h^2).
    \end{align*}
Thus, equation (\ref{eq-s11}) holds using Slutsky's Theorem.  This completes the proof.

\end{proof}

\medskip

{\bf Theorem \ref{thm5-6}} (Bias of the Estimated Counterfactual Outcome). Under the same conditions as in Proposition \ref{prop5-5}, $\hat y_{x'}$ converges to $\bar y_{x'}$  in probability, where $\bar y_{x'}$ satisfies that 
    \begin{equation} \label{eq3}  \P( Y_{x'} \leq \bar y_{x}|Z=z ) = \frac{ (2+2\delta_{p_{x}}) \P(Y_{x}\leq y| Z=z) + (\delta_{p_{x'}}- \delta_{p_{x}})  }{ 2 + 2\delta_{p_{x'}}  },
    \end{equation}
 where $\bar y_{x'}$ may not equal to the true value $y_{x'}$ and their difference is the bias.

\begin{proof}[Proof of Theorem \ref{thm5-6}]
By the property of M-estimation~\citep{vdv-1996}, $\hat y_{x'}$ converges to $\bar y_{x'}$  in probability, where $\bar y_{x'}$ is the minimizer of $$\bar R_{x'}(t; x, z, y),$$ 
 where $\bar R_{x'}(t; x, z, y)$ is the probability limit of $\hat R_{x'}(t; x, z, y)$. 
By the proof of Proposition \ref{prop5-5}, 
 \[    \bar R_{x'}(t; x, z, y) =  \frac{p_{x'}(z)}{\hat p_{x'}(z)} \bfE \left [  | Y_{x'} - t |  ~ \big | ~ Z=z \right ]  + \frac{p_{x}(z)}{\hat p_{x}(z)} \bfE \left [ \text{sign}(Y_x - y)  ~ \big | ~  Z = z\right ] \cdot t.    \]
By a similar proof of Theorem \ref{thm-1}, as $h \to 0$, $\bar y_{x'} = \arg\min_{t} \bar R_{x'}(t; x, z, y) $ satisfies that 
    \begin{equation*}  \P( Y_{x'} \leq \bar y_{x}|Z=z ) = \frac{ (2+2\delta_{p_{x}}) \P(Y_{x}\leq y| Z=z) + (\delta_{p_{x'}}- \delta_{p_{x}})  }{ 2 + 2\delta_{p_{x'}}  },
    \end{equation*}
\end{proof}
\section{Extension to Continuous Outcome} \label{app-d} 

When the treatment is continuous, we can estimate the ideal loss with the following estimator
 $$ \tilde R_{x'}(t| x, z, y) = \frac{ \sum_{k=1}^N K_h(z_k -z) \frac{K_h(x_k-x') }{p_{x'}(z_k)} | y_k - t |  }{ \sum_{k=1}^N K_h(z_k -z) }  + \frac{ \sum_{k=1}^N K_h(z_k -z) \frac{K_h(x_k - x) }{p_{x}(z_k)} \cdot \text{sign}(y_k - y) }{  \sum_{k=1}^N K_h(z_k -z) }  \cdot t, $$
which is a smoothed version of the estimator 
  \begin{align*}
        \hat R_{x'}(t| x, &z, y) ={} \frac{ \sum_{k=1}^N K_h(z_k -z) \frac{\mathbb{I}(x_k=x')}{p_{x'}(z_k)} | y_k - t |  }{ \sum_{k=1}^N K_h(z_k -z) }  + \frac{ \sum_{k=1}^N K_h(z_k -z) \frac{\mathbb{I}(x_k=x)}{p_{x}(z_k)} \cdot \text{sign}(y_k - y) }{  \sum_{k=1}^N K_h(z_k -z) }  \cdot t,
  \end{align*}
 defined in Section \ref{sec6}. In addition, by a proof similar to that of Proposition \ref{prop5}, we also can show that 
 $\tilde R_{x'}(t; x, z, y) \xrightarrow{\P} R_{x'}(t; x, z, y).$


\section{Experiment Details and Additional Experiment Results} \label{app-e}  
We run all experiments on the Google Colab platform. For the representation model, we use the MLP for the base model and tune the layers in $\{1, 2, 3\}$. In addition, we adopt the logistic regression model as the propensity model. We tune the learning rate in $\{0.001, 0.005, 0.01, 0.05, 0.1\}$. For the kernel choice, we select the kernel function between the Gaussian kernel function and the Epanechnikov kernel function, and tune the bandwidth in $\{1, 3, 5, 7, 9\}$.

In addition, we investigate the performance of our method when the rank preservation assumption is violated. Specifically, we simulate $ Y_1 = (W_y + W_{y1})\cdot Z + U_{1}$ and $Y_0 =  W_y \cdot Z/\alpha + U_{0}$ with $W_y \sim \mathcal{N}(0,I_{m})$ and $W_{y1} \sim \mathcal{N}(0,\beta I_{m})$, where $\beta$ is the hyper parameter to control the violation degree. In the added experiments, m is in {10, 40}, which is consistent with the experiments in our original manuscript, and Kendall’s rank correlation coefficient $\rho(Y_0, Y_1)$ is in {0.3, 0.5}. We generate 10 different datasets with different seeds and the experiment results are shown in Table \ref{tab:sim_rank}. The experiment results show that our method still outperforms the baseline methods, even if the rank preservation assumption is violated.
\begin{table*}[]
\centering
\caption{$\sqrt{\epsilon_{\text{PEHE}}}$ of individual treatment effect estimation on the simulated Sim-$m$ dataset, where $m$ is the dimension of $Z$.}
\resizebox{1\linewidth}{!}{
\begin{tabular}{l|ll|ll|ll|ll}
\toprule  & \multicolumn{2}{c|}{Sim-10 (Rank=0.3)}   & \multicolumn{2}{c|}{Sim-10 (Rank=0.5)}  & \multicolumn{2}{c|}{Sim-40 (Rank=0.3)}   & \multicolumn{2}{c}{Sim-40 (Rank=0.5)}
\\ \midrule
Methods   & \multicolumn{1}{c}{In-sample} & \multicolumn{1}{c|}{Out-sample} & \multicolumn{1}{c}{In-sample} & \multicolumn{1}{c|}{Out-sample} & \multicolumn{1}{c}{In-sample} & \multicolumn{1}{c|}{Out-sample} & \multicolumn{1}{c}{In-sample} & \multicolumn{1}{c}{Out-sample} \\
\cmidrule{1-9}
TARNet     & 6.92 $\pm$  0.79          & 6.96 $\pm$  0.82          & 4.38 $\pm$  0.59          & 4.38 $\pm$  0.60          & 14.91 $\pm$  0.61         & 14.80 $\pm$  0.82         & 9.33 $\pm$  0.36          & 9.33 $\pm$  0.45          \\
DragonNet  & 6.97 $\pm$  0.83          & 7.02 $\pm$  0.88          & 4.43 $\pm$  0.62          & 4.44 $\pm$  0.62          & 15.04 $\pm$  0.54         & 14.94 $\pm$  0.77         & 9.18 $\pm$  0.30          & 9.16 $\pm$  0.50          \\
ESCFR      & 6.91 $\pm$  0.69          & 6.96 $\pm$  0.76          & 4.41 $\pm$  0.65          & 4.41 $\pm$  0.68          & 15.01 $\pm$  0.58         & 14.86 $\pm$  0.82         & 9.25 $\pm$  0.18          & 9.26 $\pm$  0.31          \\
X\_learner & 7.01 $\pm$  0.81          & 7.05 $\pm$  0.85          & 5.29 $\pm$  0.44          & 4.98 $\pm$  0.48          & 15.16 $\pm$  0.64         & 15.05 $\pm$  0.88         & 9.43 $\pm$  0.29          & 9.40 $\pm$  0.42          \\
Quantile-Reg & 6.79 $\pm$  0.79          & 6.63 $\pm$  0.84          & 4.12 $\pm$  0.63          & 4.14 $\pm$  0.65          & 13.12 $\pm$  0.62         & 13.25 $\pm$  0.87         & 8.39 $\pm$  0.30          & 8.43 $\pm$  0.42          \\
Ours       & \textbf{6.00 $\pm$  1.22}* & \textbf{6.01 $\pm$  1.29} & \textbf{3.88 $\pm$  1.25} & \textbf{3.90 $\pm$  1.30} & \textbf{10.76 $\pm$  0.52}* & \textbf{10.84 $\pm$  0.43}* & \textbf{6.75 $\pm$  0.24}* & \textbf{6.80 $\pm$  0.21}* \\
\bottomrule
\end{tabular}}
\label{tab:sim_rank}
\end{table*}

\end{document}